\documentclass{article}

\usepackage{geometry}
\usepackage{amsmath,amsfonts}
\usepackage{algorithm}
\usepackage{array}
\usepackage[caption=false,font=normalsize,labelfont=sf,textfont=sf]{subfig}
\usepackage{textcomp}
\usepackage{stfloats}
\usepackage{url}
\usepackage{verbatim}
\usepackage{graphicx}
\usepackage{cancel}
 \usepackage{listings}
 \usepackage{amssymb}
 \usepackage{todonotes}
 \usepackage{underoverlap}
 \usepackage{comment}
\usepackage{stmaryrd}
\usepackage{algpseudocode}
\usepackage[toc]{appendix}

\usepackage{amsthm}

\newtheorem{theorem}{Theorem}

\newtheorem{proposition}{Proposition}

\newcommand{\nameOurAlgo}{{\scshape TEDS}}
\newcommand{\nameAlgoFord}{{\scshape IMP}}
\newcommand{\nameSDIF}{{\scshape SDIF}}
\newcommand{\nameOT}{{\scshape OT}}

\usepackage{amsmath}

\usepackage{authblk}
\usepackage[pdfborder={0 0 0}]{hyperref}

\begin{document}


\title{Deinterleaving of Discrete Renewal Process Mixtures with Application to Electronic Support Measures}

\author[1]{Jean Pinsolle}
\author[2]{Olivier Goudet}
\author[1]{Cyrille Enderli}
\author[2]{Sylvain Lamprier}
\author[2]{Jin-Kao Hao}

\affil[1]{Thales DMS France SAS, 2 Av. Jean d’Alembert, 78190 Trappes, France \\
\url{{jean.pinsolle,cyrille-jean.enderli}@fr.thalesgroup.com}}
\affil[2]{LERIA, Universit\'{e} d'Angers, 2 Boulevard Lavoisier, Angers 49045, France \\
\url{{olivier.goudet,sylvain.lamprier,jin-kao.hao}@univ-angers.fr}}




\maketitle

\begin{abstract}

In this paper, we propose a new deinterleaving method for mixtures of discrete renewal Markov chains. This method relies on the maximization of a penalized likelihood score. It exploits  all available information about both the sequence of the different symbols and their arrival times. A theoretical analysis is carried out to prove that  minimizing  this score allows to recover the true partition of symbols in the large sample limit, under mild conditions on the component processes. This theoretical analysis is then validated by experiments on synthetic data. Finally, the method is applied to deinterleave pulse trains received from different emitters in a RESM (Radar Electronic Support Measurements) context and we show that the proposed method competes favorably with state-of-the-art methods on simulated warfare datasets.
\end{abstract}


\section{Introduction}

We consider the following problem.  A sequence  $z = (z_i)_{i \in \mathbb{Z}}$  of  symbols has been observed during a time window $\llbracket 0,T \rrbracket $, with $T \in \mathbb{N}^*$. The symbols $z_i$ of the sequence arrive at successive integer time steps $t_i \in \mathbb{Z}$, with the first symbol of the observation window being indexed as $z_0$:  $0  \leq t_0 <  t_2 < \dots \leq T$.  $t$ is the sequence of arrival times observed until time $T$. Each symbol of $z$ is drawn from a finite set  $\mathcal{A}$ (alphabet).  The underlying generative model $P$ of this sequence is assumed to be a set of $m > 0$  Markov processes $G_{\Pi}(\{P^e\}_{e \in E(\Pi)})$, 
where $E(\Pi)$ is a set of emitters and $m=|E(\Pi)|$ is the number of different emitters in this set. Each $P^e$, with $e \in E(\Pi)$, 
is an independent random process for emitter $e$, generating symbols in the sub-alphabet $A_e \subset \mathcal{A}$ with their corresponding arrival times.
$\Pi = \{A_e\}_{e \in E(\Pi)}$ 
is the partition of $\mathcal{A}$ into the sub-alphabets $A_e$, 
which are assumed to be non-empty and disjoint. Given a sample $(z,t)$  from $P$, observed until time $T$, and without prior knowledge about the number of  sub-alphabets $A_e$ (number of emitters), the deinterleaving problem of interest is to reconstruct the original partition  $\Pi = \{A_e\}_{e \in E(\Pi)}$ 
of the alphabet.

This problem is motivated by an application in the context of electronic warfare, where an  airborne ESM (Electronic Support Measure) receiver is located in an environment composed of many radars. This receiver intercepts pulse trains from multiple transmitters over a common channel. Each pulse is characterized by a frequency, which can be associated with a symbol $z_i$ in alphabet $\mathcal{A}$ (using a clustering method in frequency space), and a time of arrival (TOA), which can be discretized and assigned an integer value $t_i$. Other parameters such as pulse duration, power, or  angle of arrival, may also be available, but in this work we restrict attention to the time and frequency information. In this context, the sub-alphabets $A_e$ correspond to the frequency channels used by the transmitters. They realize a partition of the total processed frequency band, which corresponds to the case where transmitters avoid mutual interference by controlling their access to the radio spectrum. The cases where different transmitters share common frequency channels is not addressed here. It corresponds to the case of non disjoint sub-alphabets which is out of the scope of this work. From the recorded pulse train, the aim of our problem is to determine how many radars are present in the environment and what signals are emitted by each radar.

The literature on deinterleaving  RADAR signals is abundant. It was first performed in the literature by calculating histograms of received time differences \cite{Davies1982AutomaticPF,mardia1989new} or by estimating the phases of the interleaved pulse trains \cite{conroy2000estimation}.  This type of method has recently been revisited in \cite{mottier_deinterleaving_2021}, who proposed an  agglomerative hierarchical clustering method combined with optimal transport distances. Such clustering methods deliver fast outcomes, but might fail in  environments with complex radar patterns, potentially causing a performance degradation.

To overcome these limitations, another stream of work is based on inferring mixtures of Markov chains \cite{hutchison_inferring_2004}.
In these methods, a clustering algorithm is first applied to group the different pulses into different clusters, or letters, which correspond to the alphabet $\mathcal{A}$. Then, in a second step, these letters are partitioned into different sub-alphabets to identify the different emitters that could have generated the observed  sequence of letters.

These methods  are mainly based on two types of pulse train modeling in the literature: Interleaved Markov Process (IMP) and Mixture of Renewal Processes (MRP).

Modeling with IMP was introduced for the first time in \cite{hutchison_inferring_2004}. This model does not directly take into account the arrival times of pulses, but rather considers their order of arrival. The overall interlaced sequence produced by the different transmitters and received in a single channel is modeled using a switching mechanism, whose role is to alternate the sequence of transmissions by the different emitters within the observed sequence. The inference of an IMP is generally carried out using non-parametric methods, by maximizing a global penalized likelihood score. When sub-alphabets are disjoint, i.e., each symbol can only be transmitted by a single transmitter, a  deinterleaving scheme has been proposed in \cite{seroussi2012deinterleaving}. The authors  proved  the convergence of this method under mild conditions on switching and component processes, and provided that all possible partitions of letters can be evaluated. However, as the exploration of all possible partitions grows exponentially with the number of different symbols, the evaluation of all candidate partitions may be  unfeasible in practice. Therefore, different search heuristics have recently been proposed to explore the partition space: an iterated local search in \cite{ford2020deinterleaving,seroussi2012deinterleaving}, and a memetic algorithm in \cite{MAAP_2023}. Meanwhile, \cite{minot_separation_2014} has proposed a variant of the IMP  model introduced in \cite{seroussi2012deinterleaving} for the case of non-disjoint sub-alphabets inferred with an expectation maximization (EM) algorithm. The asymptotic consistency of this deinterleaving scheme is not proven. 

These IMP models do not explicitly take into account the delay between the different signals in the deinterleaving scheme, but only their order of arrival. This feature makes them robust to disturbances that may occur during arrival time measurements, but does not take advantage of all available information. However, this information can be critical to the deinterleaving process because the times between different pulses generally follow regular, identifiable distributions.

Therefore, another category of methods has emerged in the literature based on inference of mixtures of Markov renewal processes. In these models, there is no switching mechanism, and the global generation model consists of a set of independent Markov renewal chains, or semi-Markov processes, which have been studied in detail in the discrete case, notably in \cite{barbu2009semi}.  Estimation of mixtures of semi-Markov chains has been carried out in \cite{cardot2019estimating} for continuous data or discrete data generated by a negative binomial distribution. The authors used an expectation maximization (EM) algorithm. Another method for deinterleaving a mixture of renewal processes has recently been proposed in \cite{8571273}. This model is designed for continuous data and relies solely on temporal information. The proposed model assumes that the number of emitters is known in advance and that there are no labels on individual pulses.


In this paper, we aim to achieve the best of both worlds, namely to take into account all available information, such as pulse characteristics and arrival times, by modeling a mixture of renewal processes as in \cite{cardot2019estimating}. We also aim to propose a consistent deinterleaving scheme, as it has been done for  IMP models \cite{seroussi2012deinterleaving}, enabling the true partition to be recovered with probability tending towards one when sufficient data is provided. We would also like to take advantage of recent advances to explore the symbol partition space proposed by \cite{MAAP_2023} and applied to the deinterlacing of RADAR pulse trains. 

This paper is organized as follows. Section \ref{sec:model} presents the renewal processes mixture model. Section \ref{sec:model_estimation} describes the estimation of the parameters and the deinterleaving method to retrieve the partition of the symbols. Section \ref{sec:exp1} presents the experimental setting and a first experimental analysis of the consistency of the score. Section \ref{sec:benchmarks} reports on empirical results of the proposed method compared to the state of the art. Section \ref{sec:conclusion} discusses the contribution and presents some perspectives for future work.

\section{Mixture of renewal processes \label{sec:model}}

The underlying generative model $P$ of a  sequence $(z,t)$  is assumed to be a set $G_{\Pi}(\{P^e\}_{e \in E(\Pi)})$ of  $m > 0$ independent Markov processes.

\subsection{Markov renewal chains \label{sec:renewal_chains}}

Given a sequence $z$ of symbols in the alphabet $\mathcal{A}$, for each process $P^e$ associated to the sub-alphabet $A_e$, we use $z[A_e]$ , or just $z^e$, to present  the sub-string of the sequence $z$ obtained by deleting all symbols not in $A_e$ and $t[A_e]$ (or just $t^e$) the corresponding sequence of arrival times.

Each process $P^e$, with $e \in E(\Pi)$, 
is modeled as a Markov renewal chain  $(Z^e,T^e)$   with:

\begin{itemize}
   
    \item the chain $Z^e$   with state space $A_e$, generating a sequence of symbols $(z^e_l)_{l \in \mathbb{Z}}$.
    \item the chain $T^e$ generating the strictly increasing arrival times $t^e_l$  of all $z^e_l$.

\end{itemize}

For each process $P^e$, we introduce also the chain $X^e = (x_l^e)_{l \in \mathbb{Z}}$, corresponding to the sojourn time in each state of the chain $Z^e$. For all $l \in \mathbb{Z}$,  $x_l^e = t^e_l - t^e_{l-1} > 0$.

In the context of Radar Electronic Support Measurement (RESM), the arrival times of the individual pulses, and consequently the interval between them, are multiples of integers. The arrival times are quantified by the digital receiver, with a time resolution corresponding to the Least Significant Bit (LSB). Thus every time measurement is multiple of this LSB. Therefore, for each letter $a$ in a sub-alphabet $A_e$,  we assume that each sojourn time is drawn in a finite  set $K^a$ of strictly positive integers. In the following, for a given emitter $e \in E(\Pi)$,  we also note $K_e=\bigcup_{a \in A_e} K^a$  as the union of sojourn time sets of all symbols from $A_e$.

Each chain $(Z^e,T^e)$ satisfies for all $l \in \mathbb{Z}$ the following markov assumption: 
\begin{equation}
\label{eq:renewal_chain}
P^e((z^e_{l},t^e_{l}) | (z^e_{<l},t^e_{<l}))  = P^e(z^e_{l}  | z^e_{l-1}) \times P^e(t^e_{l} - t^e_{l-1}| z^e_{l-1}).
\end{equation}
where $(z^e_{<l},t^e_{<l})$ is the history of emission of transmitter $e$ before event $l$. 

Thus, after a given event at $t^e_l$,  the next symbol $z^e_{l+1}$   is independent from the chosen delay $x^e_{l+1}$. Also the delay  $x^e_{l+1}$ only depends on the previously emitted symbol $z^e_l$. 

We make these assumptions in order to reduce the number of parameters to be estimated and therefore speed up the likelihood calculations, although in practice this is not always true, since some agile radar waveforms may consist of a periodically repeated synchronized time-frequency pattern. In some other cases,  \eqref{eq:renewal_chain} may hold at least approximately. We assume (and empirically confirm in our context) that the likelihood expression derived from  \eqref{eq:renewal_chain} is robust enough to produce good deinterleaving performance in practice.

Moreover, we assume that each process $P^e$ is time homogeneous. Therefore  $P^e(z^e_{l+1} | z^e_{l})$  and $P^e(t^e_{l+1} - t^e_{l} | z^e_{l} )$ are independent from $l$.

We denote by $\mathbf{P}^e = (p^e_{i,j})$ the transition matrix of $Z^e$, with for $i,j \in A_e$ 
$p^e_{i,j} = P^e(z^e_{l+1} = j | z^e_{l} = i), l \in \mathbb{N}$, and for $l \in \mathbb{N}$, $i \in A_e$, $k \in K^i$, we introduce  the quantity 
$q^{e}_i(k) = P^e(x^e_{l+1} = k | z^e_{l} = i)$. \\

\noindent
{\bf Assumption $(\mathcal{P})$.}
We assume that $p^e_{i,j} > 0$ for  $i,j \in A_e$.
\\

\noindent
{\bf Assumption $(\mathcal{Q})$.}
For each $i \in A_e, k \in K^i$, we assume that $q^{e}_i(k) > 0$.
\\

Assumptions $(\mathcal{P})$ and $(\mathcal{Q})$ allow us to establish the uniqueness of generative model representations. The proof of Theorem 1 is given in the appendix.

\begin{theorem}
\label{theorem1}
Given a generative model $P = G_{\Pi}(\{P^e\}_{e \in E(\Pi)})$, corresponding to a partition $\Pi$ of the alphabet $\mathcal{A}$, under assumptions $(\mathcal{P})$ and $(\mathcal{Q})$, if  $P =G_{\Pi'}(\{P^e\}_{e \in E(\Pi')})$ for some partition $\Pi'$, then $\Pi = \Pi'$ must hold. 
\end{theorem}

\section{Deinterleaving Scheme}
\label{sec:model_estimation}

In this section, the estimation of the different components of the proposed discrete renewal process mixture is described and
 a consistent negative log-likelihood score is derived.

\subsection{Semi-Markov Chain Estimation}
For each transmitter $e$ with sub-alphabet $A_e$, and given a sequence $(z,t)$  observed during a time window $\llbracket 0,T \rrbracket$, we define the following quantities:

\begin{itemize}
    \item $N^e_i(T)$, the number of symbols $i$ in the observed sequence $z^e=z[A_e]$ until time $T$.
    \item $N^e_{i,j}(T)$, the number of transitions from $i$ to $j$, observed in the sequence $z^e$ until time $T$.
    \item $N^e_i(k,T)$, the number of transitions  observed from $i$ to any other letter in $z^e$ until time $T$, with sojourn time equal to $k$.
    \item $N^e(T)$, the total number of symbols in $z^e$ . 
		
\end{itemize}

For a transmitter $e$, let $(z^e,t^e)=(z^e_i, t^e_i)_{i=0}^{N^e(T)-1}$ the ordered sequence of events generated by $e$ during the observation window according to  \eqref{eq:renewal_chain}. Let also $(z^e_{-1}, t^e_{-1})$, with $t^e_{-1}< 0$, be the last event generated by $e$ directly preceding the observation window. The conditional likelihood of the observed sequence $(z^e,t^e)$ during $\llbracket 0,T \rrbracket$ given $(z^e_{-1}, t^e_{-1})$ is defined as:

\begin{equation}
\label{exact_likelihood}
     \mathcal{L}_-^T(z^e,t^e)
     = P^e(z^e_0|z^e_{-1}) \frac{P^e(t^e_{0} - t^e_{-1}|z^e_{-1})}{R^e_{-1}(|t^e_{-1}| - 1)} 
      \prod_{l = 1}^{N^e(T)-1} P^e(z^e_{l}  | z^e_{l-1}) P^e(t^e_{l} - t^e_{l-1}|z^e_{l-1})  R^e_{N^e(T)-1}(u^e_T),
\end{equation}

\noindent where $R_i(.)$, the survival function of sojourn time in state $i$, is defined for $k \in \mathbb{N}$ by $R_i(k) = P(x^e_{i+1} > k | z^e_{i})$, and $u^e_T = T - t^e_{N^e(T)-1}$. 

However, $(z^e_{-1},t^e_{-1})$ is unknown, and its bayesian estimation would imply a complex marginalization, intractable in our deinterleaving setting. Rather, we rely on the following approximation, which ignores start and end 
of the observation window without biasing the process (see proposition 6 in the appendix): 

\begin{equation}
\label{eq:approach_likelihood}
    \mathcal{L}^T_{\sim}(z^e,t^e) 
      = \prod_{l = 1}^{N^e(T)-1} P^e(z^e_{l}  | z^e_{l-1}) P^e(t^e_{l} - t^e_{l-1}|z^e_{l-1})
 \end{equation}

For each transmitter $e$, we define the empirical estimators of the coefficients $p^e_{i,j}$ of the transition matrix $\mathbf{P}^e$ and  the empirical estimators of the coefficients $q^e_i(k)$ of the distributions of the sojourn times, respectively   by 
$\hat{p}^{e,\sim}_{i,j}(T):= \frac{N^e_{i,j}(T)}{N^e_i(T)}$ and 
$\hat{q}^{e,\sim}_{i} (k,T):= \frac{N^e_i(k,T)}{N^e_i(T)}$, if $N^e_i(T) \neq 0$. If $N^e_i(T) = 0$, we set $\hat{p}^{e,\sim}_{i,j}(T):= 0$ and $\hat{q}^{e,\sim}_{i} (T,k):=0$.

\begin{proposition} \label{prop:estimator}
The estimators ${\hat{p}^{e,\sim}_{i,j}(T)}$ and $\hat{q}^{e,\sim}_{i}(k,T)$ maximize  the approached log-likelihood function $\text{log}(\mathcal{L}^T_{\sim}(z^e,t^e))$.
\end{proposition}

When using the estimators $\hat{p}^{e,\sim}_{i,j}(T)$ and $\hat{q}^{e,\sim}_{i} (k,T)$, the maximum of the approached conditional likelihood of the sequence $(z^e,t^e)$ is denoted $\widehat{\mathcal{L}}^T_{\sim}(z^e,t^e)$.

Given an observed sequence $(z^e,t^e)$, generated by transmitter $e$ and observed until time $T$, we denote by $H^T(z^e,t^e) = - \text{log}(\widehat{\mathcal{L}}^T_\sim(z^e,t^e))$, the approached  maximum negative log-likelihood for transmitter $e$.

We have
\begin{align}
\label{eq:global_entropy}
H^T(z^e,t^e) = H^T_{Z}(z^e,t^e)  + H^T_{X}(z^e,t^e),
\end{align}
\noindent with $H^T_{Z}(z^e,t^e)$ the term due to the state transitions for transmitter $e$,
\begin{equation}
\label{eq:entropy_symbols}
H^T_{Z}(z^e,t^e) 
= - \sum_{i,j \in A_e} N^e_{i,j}(T) \text{log} \ \frac{N^e_{i,j}(T)}{N^e_i(T)},
\end{equation}
\noindent and $H^T_{X}(z^e,t^e)$ the term due to the distribution of sojourn time for transmitter $e$:
\begin{align}
\label{eq:entropy_time}
H^T_{X}(z^e,t^e)  
= - \sum_{i \in A_e}  \sum_{k \in K^i} N^e_i(k,T) \ \text{log} \ \frac{N^e_i(k,T)}{N^{e}_{i}(T)}.
\end{align}

By convention in \eqref{eq:entropy_symbols}, given $T > 0$, if for some $i$ and $j$, we have $N^e_{i,j}(T) =0$, i.e. there is no  transition from $i$ to $j$ in the observed sub-sequence $z^e$ until time $T$, then we set $$N^e_{i,j}(T) \text{log} \ \frac{N^e_{i,j}(T)}{N^e_i(T)} := 0.$$  In the same way, in  \eqref{eq:entropy_time}, if for $i \in A^e$, and $k \in K^i$, $N^e_i(k,T) = 0$, we set $$N^e_i(k,T) \ \text{log} \ \frac{N^e_i(k,T)}{N^{e}_{i}(T)} := 0.$$

\subsection{Global model estimation and deinterleaving scheme \label{sec:deinterleaving}}
As all transmitters are independent, the  \textit{global} approached likelihood of the  sequence $(z,t)$ observed until time $T$, and related to a partition of symbols $\Pi = \cup_{e \in E(\Pi)} A_e$, is given by

\begin{equation}
   \mathcal{L}^T_{\sim,\Pi}(z,t) =  \prod_{e \in E(\Pi)}   \mathcal{L}^T_{\sim}(z^e,t^e).
\end{equation}

We use $H^T_{\Pi}(z,t) = - \text{log} \ \widehat{\mathcal{L}}^T_{\sim,\Pi}(z,t)$ to denote the corresponding  empirical entropy. Thus, we have,

\begin{equation}
\label{eq:independance}
 H^T_{\Pi}(z,t) =  \sum_{e \in E(\Pi)}  H^T(z^e,t^e),
\end{equation}

 \noindent with $H^T(z^e,t^e)$
given by Equation \eqref{eq:global_entropy}.

Given a sequence $(z, t)$ observed during the time window $\llbracket 0,T \rrbracket$, the proposed deinterleaving scheme corresponds to finding the partition 
minimizing a penalized entropy 
score function of the form: 
 
\begin{equation}
\label{eq:score_global}
    C^T_{\Pi}(z,t) =  H^T_{\Pi}(z,t) + \gamma   m  \text{log} (n),
\end{equation}

\noindent {with $\gamma$} a non-negative (penalization) constant,  $m$ the number of transmitters and $n$ the size of the observed sequence $z$ until time $T$. Such penalization in $\log(n)$ \cite{gideon} aims to take into account the fact that the number of parameters, related to the number of transmitters, varies with the considered partition. Given a candidate partition $\Pi$, if simultaneous emission of different symbols from the same alphabet $A_e$ occurs in the observed sequence, the partition $\Pi$ is considered incompatible with the process that generated the sequence (the same process cannot emits several symbols at the same time). In that case, 
we arbitrarily set $C^T_{\Pi}(z,t):=+\infty$ to discard that wrong  candidate partition. \\

\noindent
{\bf Assumption $(\mathcal{K})$.} Let $K_e = \bigcup_{z \in A_e} K^z $ be the set of all sojourn times for transmitter $e$. For all transmitter   $e \in E(\Pi)$, associated to the underlying generative process $G_{\Pi}(\{P^e\}_{e \in E(\Pi)})$, 
the greatest common divisor (gcd) of the set $K_e$ is equal to 1.\\

Under this additional assumption, we show with the following theorem that the deinterleaving scheme allows to retrieve the \textit{true} partition when $T$ goes toward infinity,  with an exhaustive search in the space of all possible partitions of  the alphabet $\mathcal{A}$. This proof is given in the appendix. Hypothesis $(\mathcal{K})$ is quite relevant for radar emitters: to resolve ambiguities on pulse repetition interval range measurement, pulse repetition intervals are often taken purposely coprime which in turn translates into coprimed sojourn time.

\begin{theorem}
\label{theorem:convergencemodel}

     Let  $P_\Pi = G_{\Pi}(\{P^e\}_{e \in E(\Pi)})$ and let $\Pi^{'}$ be a partition of $\mathcal{A}$ such that $\Pi^{'} \neq \Pi$. Then, if $P_\Pi$ verifies  assumptions $(\mathcal{Q})$, $(\mathcal{K})$ and $(\mathcal{P})$, we have 
    \begin{equation}
       \underset{T \rightarrow \infty}{\text{lim}} \mathbb{P}_{(z,t) \sim P_\Pi} \left( \frac{1}{T} ( C^T_{\Pi'}(z,t) - C^T_{\Pi}(z,t)) >  0\right) =1.
    \end{equation}
    
    This is also true for sequences  $(z,t)$ sampled from $P_{\Pi|\Pi'}$, with $P_{\Pi|\Pi'}$ a process restricting sequences from $P_{\Pi}$ to the subset of possible sequences for $P_{\Pi'}$. %
\end{theorem}

The second part of the theorem  shows that the result not only holds due to the incompatibility of partitions. Even if, for instance, no simultaneous emission of different symbols occurs in the observed sequence (which can be rather unlikely for some scattered sets of emission delays), our deinterleaving process is able to identify the true partition if the observed sequence is sufficiently long.

This result is valid for all non-negative values of the penalization constant $\gamma$ in  (\ref{eq:score_global}), and in particular when $\gamma = 0$, as it is experimentally confirmed in Section \ref{sec:exp1} using synthetic data generated according to the model described in Section \ref{sec:model}. However, we experimentally observe in practice, as shown in Section \ref{sec:benchmarks}, that using $\gamma > 0$  can improve the results when applying the deinterleaving scheme in the ESM domain, whose incoming data does not exactly match  the proposed model.

\subsection{Solving the Combinatorial Problem in the Space of Partitions}
\label{sec:opti}

Given an alphabet $\mathcal{A}$ and a sample $(z,t)$  observed until time $T$, finding the partition $\Pi$ minimizing the score $C^T_{\Pi}$ given by  \eqref{eq:score_global} requires to solve a combinatorial problem in the search space of the alphabet partitions given by

\begin{equation}
\Omega_{\mathcal{A}}  =   \{\Pi | \mathcal{A} = \bigcup_{e \in E(\Pi)} A_e   \wedge  \forall (e,e') \in E(\Pi)^2, e \neq e' \implies A_e \cap A_{e'} = \emptyset\} . 
\end{equation}

The size of this search space can be calculated exactly with  the Bell number:
\begin{equation}
\label{eq:Bell}
B_{|\mathcal{A}|} = \sum_{k=0}^{|\mathcal{A}|} \frac{1}{k!} \sum_{i=0}^{k} (-1)^i \binom{k}{i} (k-i)^{|\mathcal{A}|}.
\end{equation}

For a small number of letters (less than 10), as the size of $\Omega_{\mathcal{A}}$ remains reasonable (less than $B_{10} = 115975$), 
we can use an exhaustive search in the space of all partitions. This can be done efficiently using the algorithm proposed in \cite{knuth1973art}.

However, since this space search grows exponentially with the number of letters, an exhaustive search is in general not feasible in a reasonable amount of time.

Therefore, when the number of letters is greater than 10, we propose to use an heuristic to find the best partition $\Pi$ in the huge search space $\Omega_{\mathcal{A}}$.
We adapted the recent memetic algorithm MAAP, which was previously used for a deinterleaving problem in \cite{MAAP_2023}, seen as a particular grouping problem. 
Like MAAP, our proposed algorithm, called \nameOurAlgo\ for \textit{Temporal Estimation Deinterleaving Scheme}, uses a population of two different candidate solutions (partitions). As described in Algorithm \ref{alg:maap}, it alternates between two phases during the search. In the first phase, a local search procedure (called TabuAP) is performed to improve the partitions in the population. In the second phase, crossovers (GLPX) are performed between partitions to generate new candidate partitions, which are further improved by the TabuAP procedure in the next generation. The algorithm is efficiently implemented with incremental evaluation techniques, taking advantage of the fact that the likelihood score of each sub-alphabet $A_i$ can be calculated independently. The appendix  gives more details on the implementation of the TabuAP and GLPX procedures.

\begin{algorithm}

\caption{TEDS -Temporal Estimation Deinterleaving Scheme (inspired by MAAP)}
\label{alg:maap}
\begin{algorithmic}[1]
\State \textbf{Input}: sequence $(z,t)$  observed until time $T$, with $z$ drawn from alphabet $\mathcal{A}$.
\State \textbf{Output}: Best partition $\Pi_{best}$ of $\mathcal{A}$ found so far.
\State $\Pi_1, \Pi_2, \Pi_{best} \leftarrow$ random draws in partition space  $\Omega_{\mathcal{A}}$
\While {stop condition is not met do}
    \State $\Pi_1'\leftarrow \text{TabuAP}(\Pi_1,(z,t))$
    \State $\Pi_2'\leftarrow \text{TabuAP}(\Pi_2,(z,t))$
    \If {$C^T_{\Pi_1'}(z,t) < C^T_{\Pi_{best}}(z,t)$}
        \State  $\Pi_{best} \leftarrow \Pi_{1}'$
    \EndIf
    \If {$C^T_{\Pi_2'}(z,t) < C^T_{\Pi_{best}}(z,t)$}
       \State $\Pi_{best} \leftarrow \Pi_{2}'$
    \EndIf
    \State $ \Pi_1 \leftarrow \text{GLPX}(\Pi_1', \Pi_2',(z,t))$ 
    \State $ \Pi_2 \leftarrow \text{GLPX}(\Pi_2', \Pi_1',(z,t))$  
\EndWhile
\State return $\Pi_{best}$
\end{algorithmic}
\end{algorithm}

 \section{First experimental analysis \label{sec:exp1}}

 The first experiment studies the ability of the deinterleving scheme to retrieve the \textit{true} partition, when the data are generated according to the \textit{ideal} framework described in Section \ref{sec:model}, with a collection of independent renewal processes, and when it is possible to perform an exhaustive search in the space $\Omega_{\mathcal{A}}$ of the partitions. We first describe the data generating process used for these experiments. Then we present an experimental analysis of the consistency of the score.

 \subsection{Synthetic dataset generation}
 \label{sec:datageneration}
 
 The datasets are based on synthetic sequences $(z,t)$
of size $n$ generated by a set of $m$ renewal processes  $ G_{\Pi}(\{P^e\}_{e \in E(\Pi)})$.  
Given a length $n$ and an alphabet $\mathcal{A}$, a \textit{scenario}, i.e a sequence $(z,t)$, is generated as described below. Each parameter is randomly drawn from uniform distributions on the specified interval or space.
 
 \begin{itemize}
     \item From an alphabet $\mathcal{A}$, a ground truth partition  $\Pi_{truth} = \{A_e\}_{e \in E(\Pi_{truth})}$ of the alphabet $\mathcal{A}$  is drawn in the search space  $\Omega_{\mathcal{A}}$. $m = |E(\Pi_{truth})|$ is the number of transmitters associated with this partition.
     \item For each transmitter  ${e \in E(\Pi_{truth})}$,  
     with corresponding sub-alphabet $A_e$, a probabilistic transition matrix $\mathbf{P}^e = (p^e_{i,j})$ of size $|A_e| \times |A_e|$ is drawn with non-zero coefficients in order to satisfy assumption $(\mathcal{P})$.
    \item For each symbol $a$ in $\mathcal{A}$, a number of sojourn time states is drawn in $\llbracket 1, K \rrbracket$, then different sojourn times $k_a$  are  drawn in the interval $\llbracket 1, L \rrbracket$, such that assumption $(\mathcal{K})$ is verified.
    $K$ and $L$ are two  hyper-parameters of the generator, set up to the values  $m+1$ and $|\mathcal{A}| + 1$ respectively (arbitrary choice). Then, for each symbol $a$ in $\mathcal{A}$, a probability $q_{a}(k_a)$  of  each sojourn time $k^a$, is drawn such as   $\sum_{k^a \in K^a} q_{a}(k_a) = 1$, and $\forall k^a \in K^a, q_{a}(k_a) > 0$, in order to meet assumption $(\mathcal{Q})$.
     \item An initial state symbol $z^e_0$ of each process $P^e$ is drawn in $A_e$.
   \item For each transmitter $P^e$ a sequence $(z^{e}_l,t^{e}_l)_{l \in \mathbb{N}}$ is generated using the initial state $z^e_0$, the transition matrix $\mathbf{P}^e$ and the distribution of sojourn times for each symbol $a \in A_e$.
     \item The $m$  sequences $(z^{e}_l,t^{e}_l)$ are merged to build a sequence $(z,t)$ of size $n$, with increasing times of arrival $0 \leq t_0 \leq t_1 \leq  \dots \leq t_{n-1} \leq T$.
 \end{itemize}
 
Note that the number of transmitters is drawn between 1 and the number of symbols $|\mathcal{A}|$ (see first point), but it is not equivalent to a uniform random draw in $\llbracket 1,|\mathcal{A}| \rrbracket$ since it depends on the topology of the partition space. Still, the probability to draw a partition with a high number of transmitters $m$ increases as the number of symbols $|\mathcal{A}|$ increases.

 Using this synthetic dataset generator, denoted as  $G(\mathcal{A},n)$,\footnote{The generated datasets are available at \url{https://github.com/JeanPinsolle/renewal_processes}.}  we first study different scenarios with different sizes $n$ and a low number of symbols. This allows us to experimentally verify that the ground truth  partition $\Pi_{truth}$ can be retrieved with probability one, when $n$ goes toward infinity and  when an exhaustive search in the space of all partitions is performed.

\begin{figure}[!t]
\centering
 \includegraphics[width=0.8\textwidth]{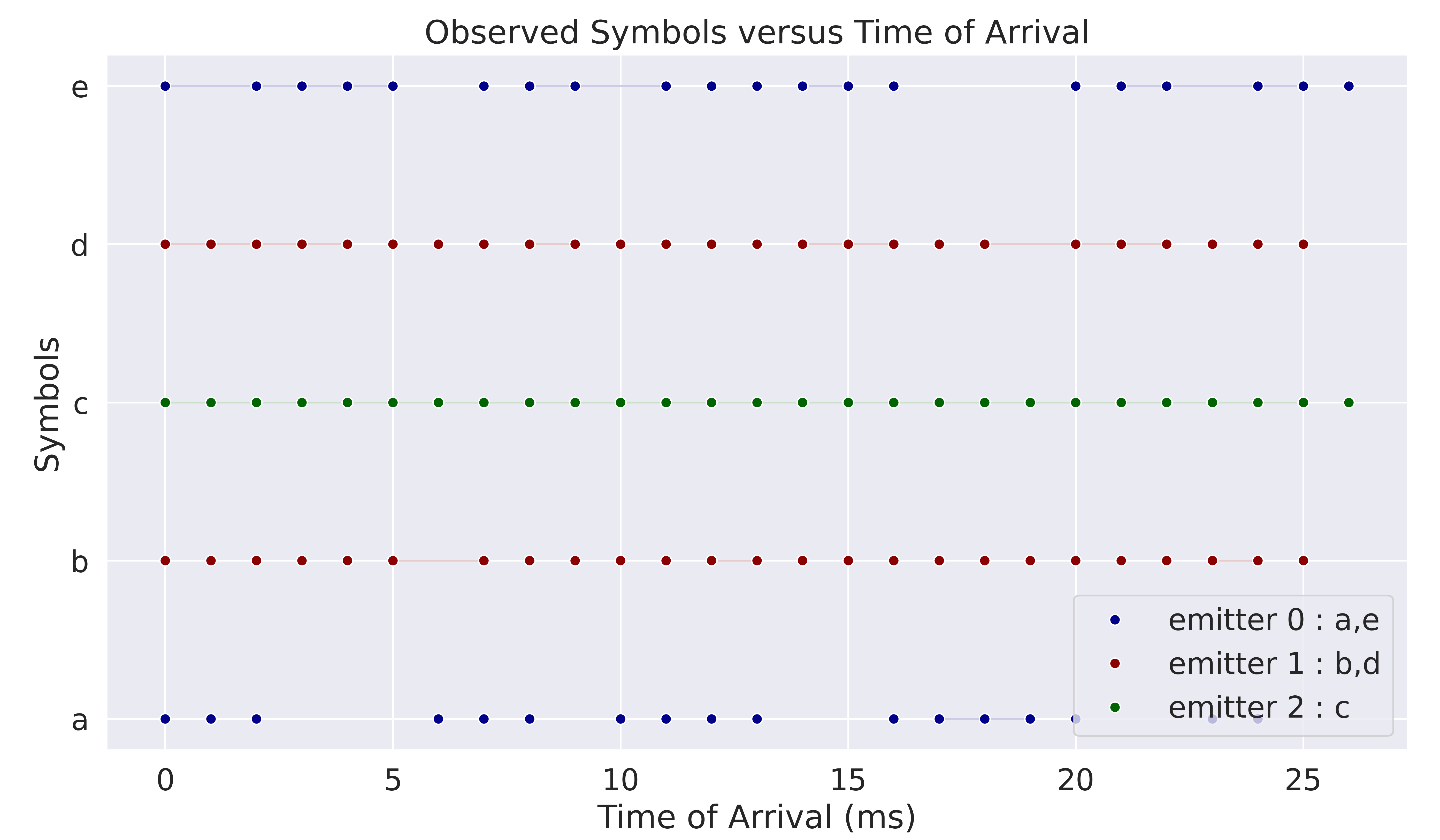}
\caption{Display of data generated for a scenario of size 200, with 5 different symbols coming from 3 different emitters. The y-axis represents the symbols, the x-axis their arrival times.  The ground truth is $\Pi_{truth} = \{a,e\} \cup \{b,d\} \cup \{c\}$: emitter 0 in blue emits symbols $a$ and $e$, emitter 1 in red emits symbols $b$ and $d$, emitter 2 in green emits symbol $c$.} \label{fig:exemple_simus}
\end{figure}

Figure \ref{fig:exemple_simus} illustrates an example of generated scenario with $|\mathcal{A}| = 5$ and $n = 200$. Here the ground truth is $\Pi_{truth} = \{a,e\} \cup \{b,d\} \cup \{c\}$, there are 3 different transmitters, each distinguished by its color.

\subsection{Empirical verification of the consistency of the score}
 \label{expe_convergence}
 
In order to analyse the robustness of our deinterleaving scheme, we study different configurations with a number of symbols $|\mathcal{A}| \in \{3,5,7,9\}$ and a length $n$ ranging in the interval $\llbracket 0, 5000 \rrbracket$. For each configuration $(|\mathcal{A}|,n)$,
1000 different datasets $(\mathcal{A}, \Pi_{truth} , z,t)$ are independently generated.

For each of these  datasets, an exhaustive search of the partition $\hat{\Pi}$ minimizing the score $C^T_{\hat{\Pi}}(z,t)$,  given by Equation \eqref{eq:score_global} and computed with $\gamma = 0$, is performed.\footnote{The number of candidate partitions are respectively  $5, 52, 877, 21147$ when the numbers of letters are respectively $|\mathcal{A}| = 3,5,7,9$   (cf. Equation \eqref{eq:Bell}). It is therefore possible to enumerate all the partitions in a reasonable amount of time in this case.} At the end of this process, if $\hat{\Pi}$ is equal to  $\Pi_{truth}$ (up to a permutation of the different sub-alphabets), we consider it a success, otherwise we consider it a failure.

Figure  \ref{fig:graphe_convergence} shows the average success rate of this procedure computed for 1000 independent scenarios for each configuration $(|\mathcal{A}|,n)$. 

These results confirm the convergence of the model even without penalization on the score (when $\gamma=0$), as stated by Theorem \ref{theorem:convergencemodel}. Indeed, we observe on this figure that for all  scenarios with different numbers of symbols, that the success rate converges toward $100\%$ when the sequence length $n$ of the observed data $(z,t)$ increases. 

We also observe a decrease in convergence speed as alphabet size increases. For alphabet sizes of $3,5,7,9$, the success rate reaches $99 \%$ for sequence lengths $n$ of $460, 875, 1920, 2395$, respectively.  This can be explained by the greater number of parameters to be estimated as the alphabet size increases.

\begin{figure}[h]
\centering
 \includegraphics[width=0.8\textwidth]{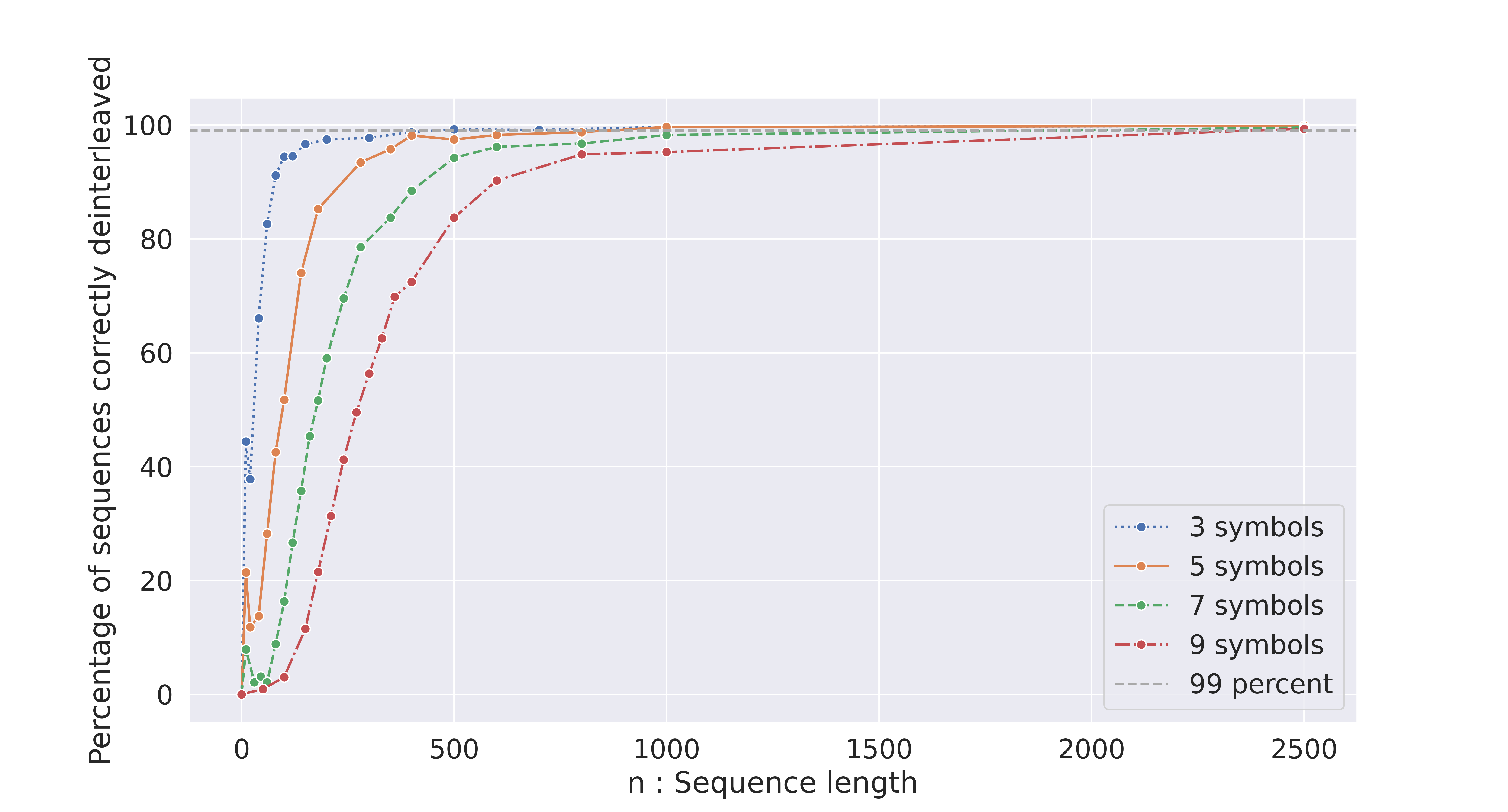}
 \caption{Average success rate of the proposed deinterleaving scheme when an exhaustive search in the partition space is performed and displayed for different numbers of symbols and different sequence sizes $n$. The horizontal line represents the $99\%$ threshold of correctly deinterlaced sequences.  \label{fig:graphe_convergence}}
\end{figure}

\section{Experimental validation}
\label{sec:benchmarks}

The goal of this section is to evaluate the performance of the proposed algorithm \nameOurAlgo\ in comparison with state-of-the-art algorithms of the literature.  The  \nameOurAlgo\ algorithm  corresponds to the evaluation of the score presented in Section  \ref{sec:deinterleaving} coupled with the memetic algorithm mentioned in Section \ref{sec:opti} for the exploration of the partition space. We first describe  the baseline algorithms, followed by  their hyper-parameter settings used in the experiments. We then report the experimental results obtained on the synthetic datasets with 5, 10, 20 and 50 symbols coming from the generator presented in the last section, followed by the results on datasets coming from an ESM data generator, which simulates realistic situations with mobile radar warning receivers.
 
\subsection{State-of-the-art algorithms  }

For comparison with the state of the art, the following algorithms have been re-implemented in Python 3 from the authors' original papers:

\begin{itemize}
 \item The improved \nameSDIF\ algorithm from \cite{liu2017improved}. SDIF is an algorithm first proposed in \cite{mardia1989new}, which is a well-known reference in the RADAR community. It aims to identify different letter sub-sequences based on the calculation of histograms of observed sojourn times. 
 
\item OT: the deinterlacing algorithm proposed in \cite{mottier_deinterleaving_2021}.
In this algorithm, a matrix of  optimal transport distance is first computed between all sub-sequences of symbols.\footnote{We used the POT Python library for this implementation  \url{https://pythonot.github.io/}.} Then, according to this distance matrix,  a hierarchical clustering method is applied to group the different sequences.
\item \nameAlgoFord: the deinterleaving scheme proposed in \cite{seroussi2012deinterleaving} with the exploration of the partition space developed by \cite{ford2020deinterleaving}. In this approach, the  underlying generative model of the  sequence is assumed to be an \textit{interleaved Markov process} $P = \mathcal{I}_{\Pi}(P_1,\dots,P_m;P_w)$, where $m > 0$ is the unknown number of different transmitters, $P_i$ is an independent component random process for transmitter $i$, generating symbols in the sub-alphabet $A_i \subset \mathcal{A}$,  $P_w$ is a random switch process over transmitters. The vector $\textbf{k} = (k_1, \dots, k_m; k_w)$ corresponds to the different orders of the  components and switch Markov processes. The authors proposed a deinterleaving scheme, given by minimizing a global cost function:

\begin{equation}
\label{pml_estimation}
(\hat{\Pi}, \hat{\textbf{k}}) = \underset{(\Pi,\textbf{k})}{argmin}\  C^I_{(\Pi,\textbf{k})}(z),
\end{equation}

\noindent  with $C^I_{(\Pi,\textbf{k})}(z)$ a global penalized  maximum likelihood (ML) entropy score  of the sequence $z$ of size $n$ (sorted by increasing times of arrival $t$) under the IMP model defined as  
\begin{equation}
\label{eq:pml_ford}
    C^I_{(\Pi,\textbf{k})}(z) =  \sum_{i=1}^m \hat{H}_{k_i}(z[A_i]) + \hat{H}_{k_w}(\sigma_{\Pi}(z)) + \beta \kappa \log n,
\end{equation}

\noindent with $\hat{H}_{k_i}(z[A_i])$ the ML entropy of each process $P_i$, $\hat{H}_{k_w}(\sigma_{\Pi}(z))$ the ML entropy of the switch process, $\beta$ a constant and $\kappa$ the number of free parameters in the model. 
\end{itemize}

\subsection{Parameter settings  }

In \nameOurAlgo,  the  penalization parameter $\gamma$  in Equation (\ref{eq:score_global}) is set to  $0$ for the first set of experiments on synthetic data (Section \ref{sec:benchmarks_synthethic}), while it is calibrated on a train set to the value of $19$ for the ESM dataset considered in Section \ref{sec:benchmarks_warfare}. 
For the tabu search procedure, the tabu tenure parameter $\alpha$ is set to the value of 0.6 (as in \cite{MAAP_2023}). The maximal number of iterations for each tabu search is set to 50. Table \ref{table:parameters}  summarizes the parameter setting for our algorithm and the different competitors.  For each algorithm and each scenario, a maximum computation time of one hour is retained.

\begin{table}[!h]
\centering
\caption{Parameter settings}
\begin{tabular}{l|l|l}
Parameter & Description & Value\\
\hline
\hline
\nameOurAlgo & & \\
$nb_{iter}$ & Nb iterations of local search  & 50\\
$\alpha$ & Tabu tenure parameter  & 0.6\\
$\gamma$ & ML Penalization parameter & 0 (sec \ref{sec:benchmarks_synthethic}), 19 (sec \ref{sec:benchmarks_warfare})\\
\hline
\hline
\nameAlgoFord\ \cite{ford2020deinterleaving} & & \\
$nb_{iter}$ & Nb iterations local search & 50\\
$r$ & Radius of random jump & 2\\
$\beta$ & ML Penalization parameter & 0.5 (sec \ref{sec:benchmarks_synthethic}), 0.1 (sec \ref{sec:benchmarks_warfare})\\
\hline
\hline
\nameOT\ \cite{mottier_deinterleaving_2021} & & \\
$s$ & Threshold silhouette score & -0.1 (sec \ref{sec:benchmarks_synthethic}), 0.2 (sec \ref{sec:benchmarks_warfare}) \\
\hline
\hline
\nameSDIF & & \\
$\epsilon$ & Precision for histograms& 0.05 \\
$X$& Threshold peaks histograms & 0.9 \\
$c$ & Threshold PRI transform & 0.5 \\
\end{tabular}
\label{table:parameters}
\end{table}

\subsection{Experiments on synthetic datasets}
\label{sec:benchmarks_synthethic}

We first consider the synthetic datasets presented in Section \ref{sec:datageneration} with alphabets $\mathcal{A}$  of size 5, 10, 20 and 50   and observed sequences of size $n=500, 2000, 5000$. Figure \ref{fig:boxplots_simus} displays with different boxplots the distribution of the V-measure scores obtained for 100 independent scenarios by each compared method (\nameSDIF, \nameOT, \nameAlgoFord,  and \nameOurAlgo)  for each configuration $(|\mathcal{A}|,n)$. 

 The V-measure score is a classical metric used to evaluate the quality of a partitioning result. It corresponds to the harmonic mean of the measures of homogeneity and completeness of the  partition of symbols  in comparison with the ground truth \cite{Rosenberg2007VMeasureAC}. The higher the scores, the better the results.
 
 \begin{figure}
 \centering

     \includegraphics[width=0.8\textwidth]{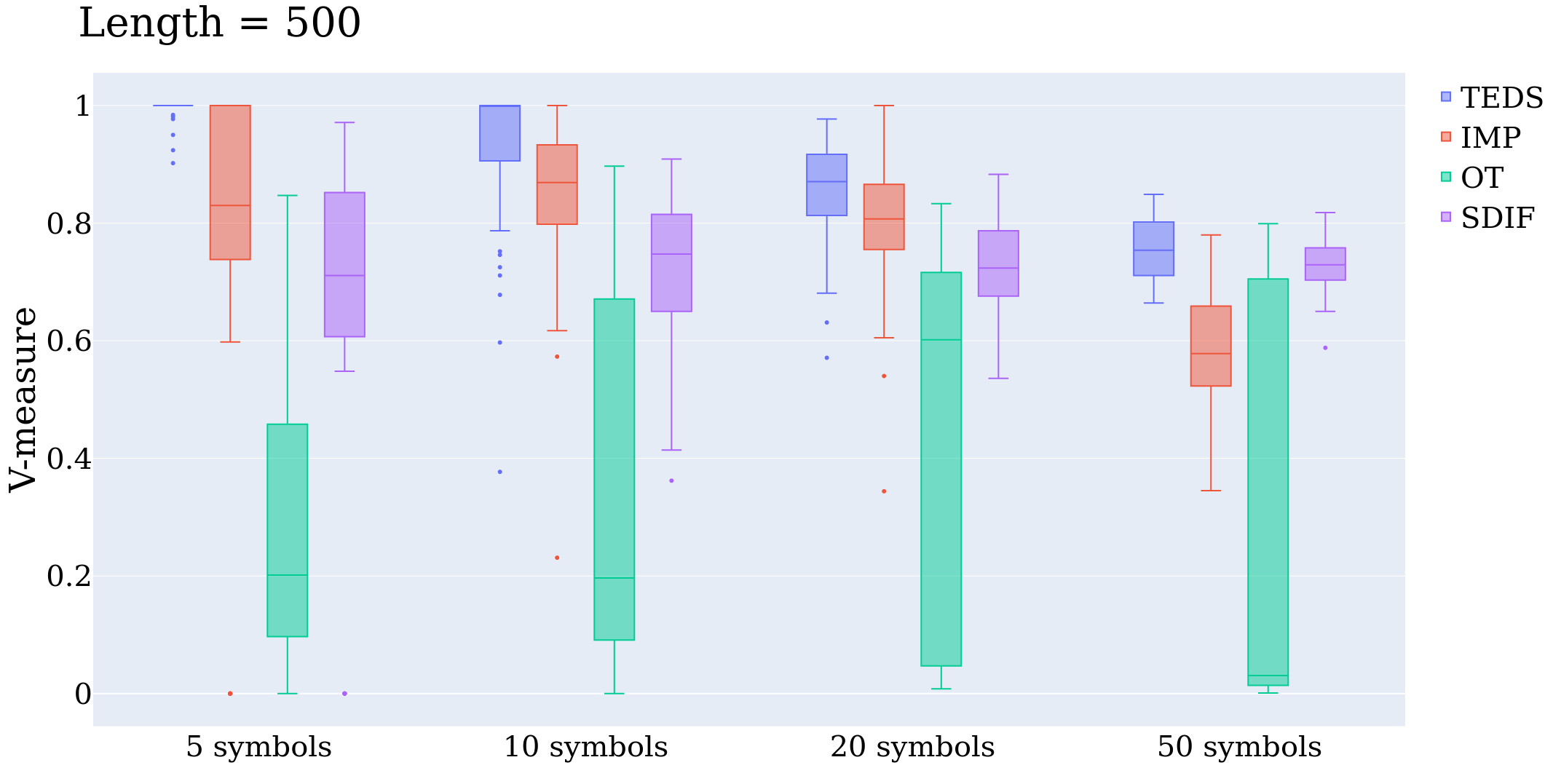}

     \includegraphics[width=0.8\textwidth]{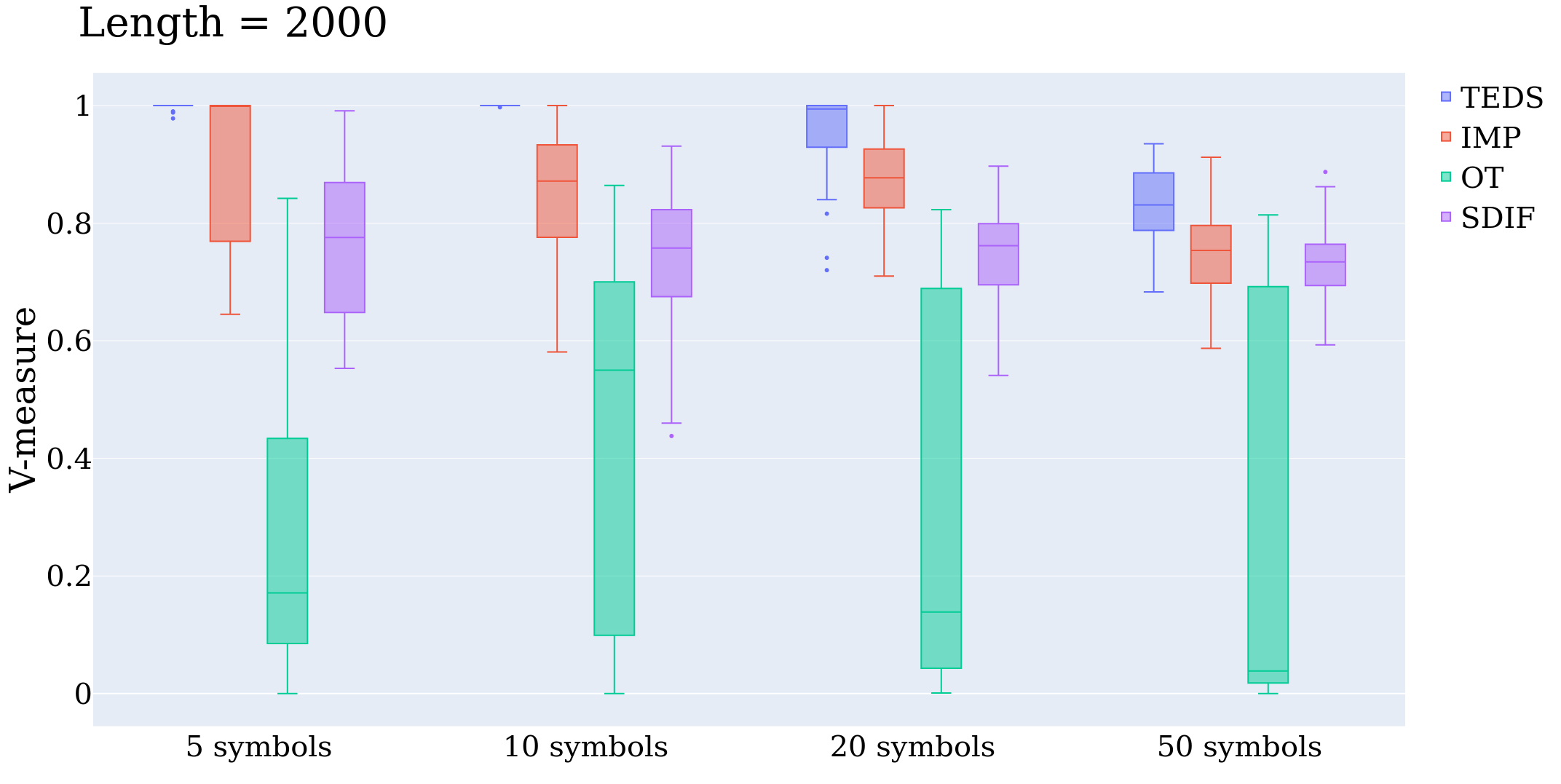}
     
     \includegraphics[width=0.8\textwidth]{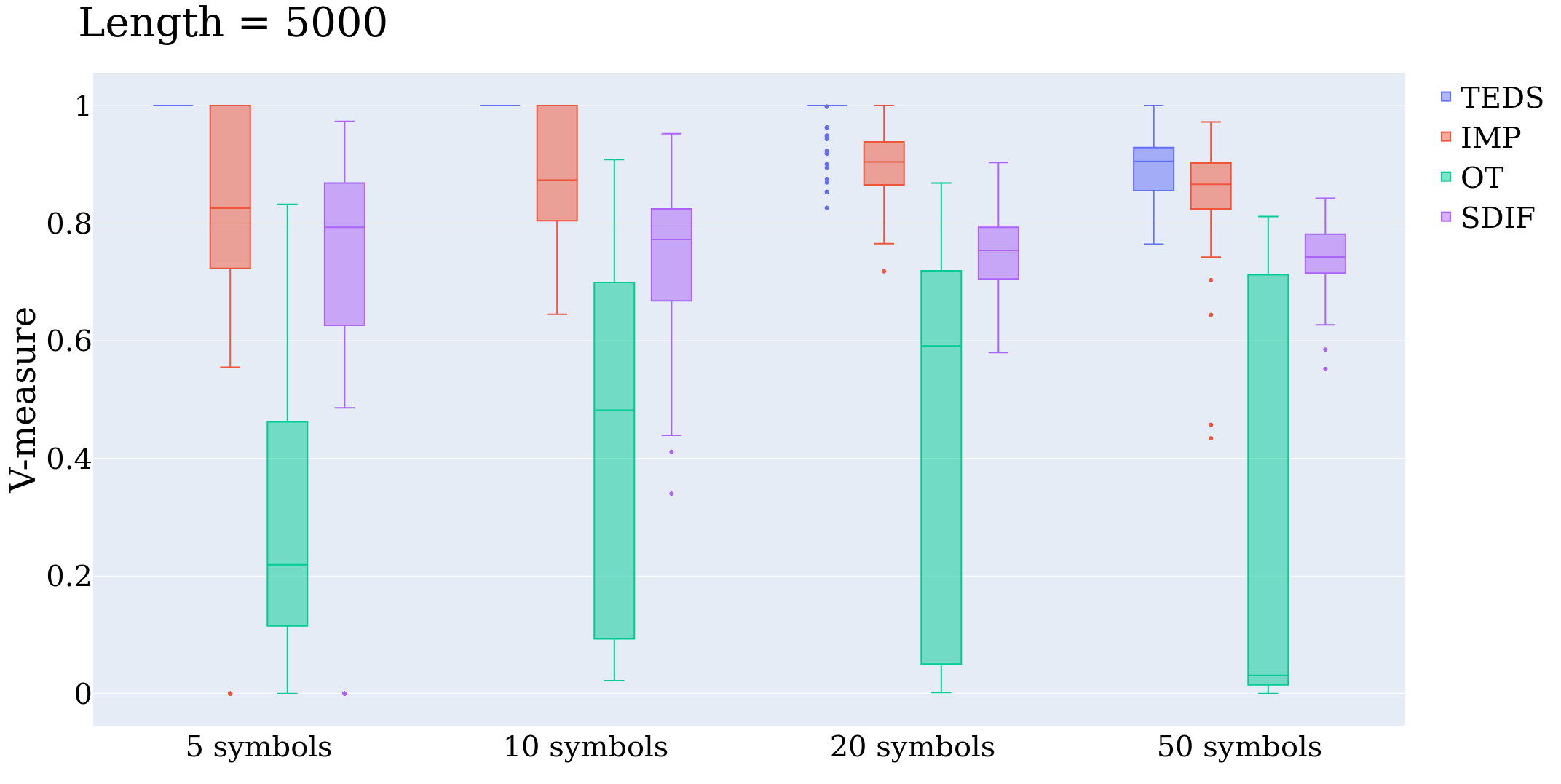}
     
   \caption{Comparison of the deinterleaving algorithms on synthetic data of size $n=500$, 2000 and 5000. Tests were made for 5, 10, 20 and 50 symbols. Each box represents the distribution of the V-measure scores obtained by the compared algorithms and computed for 100 independent scenarios.  \label{fig:boxplots_simus}}
 \end{figure}

First we observe that \nameOurAlgo\ (in blue) always returns the ground truth partition for the scenarios with 5 and 10 symbols and sequences with 5000 data points (V-measure scores always equal to 1). 
For other configurations, the number of data points is not sufficient to guarantee a consistent estimation of the score associated to each partition. Unsurprisingly, the worst results are obtained for scenarios with the highest number of symbols, $|\mathcal{A}|=50$, and the lowest numbers of data points, $n=500$. 

Overall, we observe that \nameOurAlgo\ (in blue) dominates all the competitors for all configurations.  The second best algorithm is \nameAlgoFord\ (in red), which also uses a deinterleaving scheme based on a Markov representation of the data.  Algorithms that are not based on a Markov model return partitions far from the ground truth, as evidenced by the low V-measure scores obtained by \nameSDIF\ and \nameOT\ for the different scenarios.

Regarding the computation time required to find the best partition,  \nameOurAlgo\ and \nameAlgoFord\ require the highest computing time, because  they need to compute thousands of entropy scores during the exploration of the partition space.

\subsection{Electronic Warfare experiments}
\label{sec:benchmarks_warfare}

In this section we present an application on datasets coming from an Electronic Warfare data generator which simulates realistic situations with an airborne  ESM (Electronic Support Measure) receiver in an environment composed of many radars whose number and positions can be selected in the simulation. One configuration corresponds to a random draw in a realistic radar library and a random draw in their relative phasing.
For each configuration, we generate a dataset $\mathcal{D}$ consisting in a sequence of intercepted pulse parameters, including their corresponding frequency (CF) and time of arrival (ToA). The \textit{ground truth} partition  $\Pi_{truth}$ (i.e., the association of each pulse to each transmitter) is given by the simulator but is assumed unknown. The objective is then to retrieve $\Pi_{truth}$ from the data.

\subsubsection{Data description}

Data are composed of $100$ scenarios generated with the ESM simulator. One scenario, illustrated in Figure \ref{exemple_radar}, corresponds to a manually selected window lasting few seconds in a simulation of several minutes. The windows $\llbracket 0, T \rrbracket$ during which the pulses are observed for each scenario,  were selected regarding two criteria: two different transmitters cannot transmit in the same frequency (because otherwise we are in the case of non-disjoint alphabets) and all transmitters cannot cease to transmit in the time window in order to meet as much as possible the hypothesis $(\mathcal{P})$ and $(\mathcal{Q})$ made for the data generating process (see Section \ref{sec:model}). Scenarios were generated with or without missing pulses and with or without noises, the missing rate can reach $60\%$ and missing pulses are randomly drawn according to a uniform law. The length of the sequences observed in the different scenarios ranges from $90$ to $8600$, and the number of transmitters varies from 1 to 11.

\subsubsection{Preprocessing step}
Contrary to synthetic data, the number of symbols is not fixed and is deduced from the observed data with the method explained below.

A preprocessing of the data is first performed to obtain the alphabet $\mathcal{A}$ from the dataset, as it is done in \cite{ford2020deinterleaving}. It consists in clustering pulses with the DBSCAN algorithm \cite{ester1996density} based on their frequency. Then, each  cluster obtained is associated with a symbol in $\mathcal{A}$. We then obtain  a sequence of frequency symbols $z$ of size $n$ with their corresponding arrival times $t$.

The $\epsilon$-neighborhood parameter of DBSCAN corresponds to our precision parameter and is a fixed  number of the order of the frequency measured in MHz. 
After this pulse clustering into the different groups (symbols), we obtain a number of symbols ranging from $7$ to $53$ in the different scenarios.

Figure \ref{exemple_radar} illustrates an example of  data generated with the ESM simulator. The first plot represents the frequency of the observed pulses (y-axis) according to their time of arrival (x-axis) as they are fed into the algorithm. The second plot shows the result of the pre-processing step, where the x-axis represents the arrival times, but the y-axis shows the symbols made with frequency values (frequency values are not considered anymore: only symbol grouping remains). 
Ground truth is represented, with four emitters and each one associated with one color.

\begin{figure}[!t]
\centering
\includegraphics[width=0.8\textwidth]{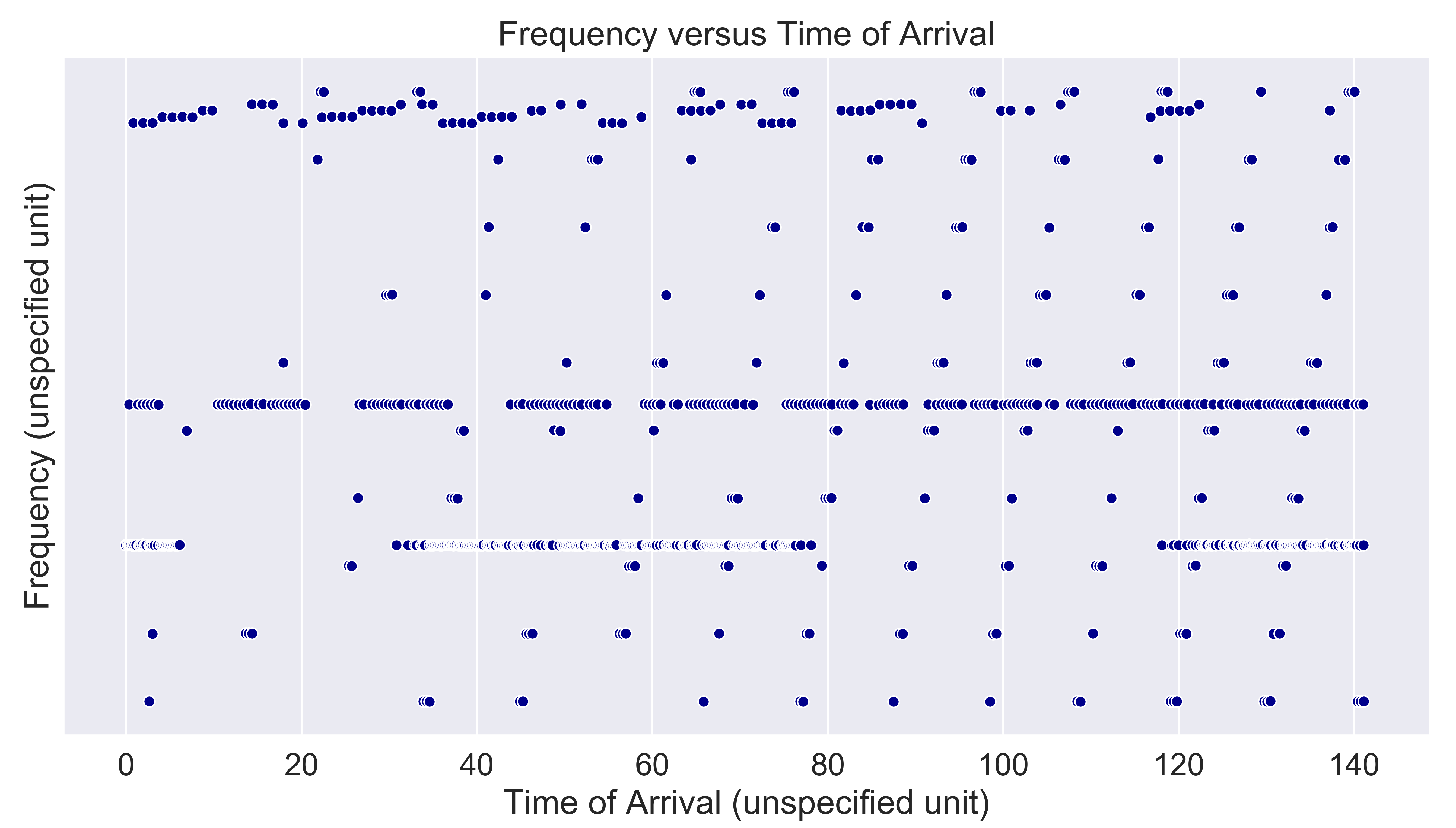}
\includegraphics[width=0.8\textwidth]{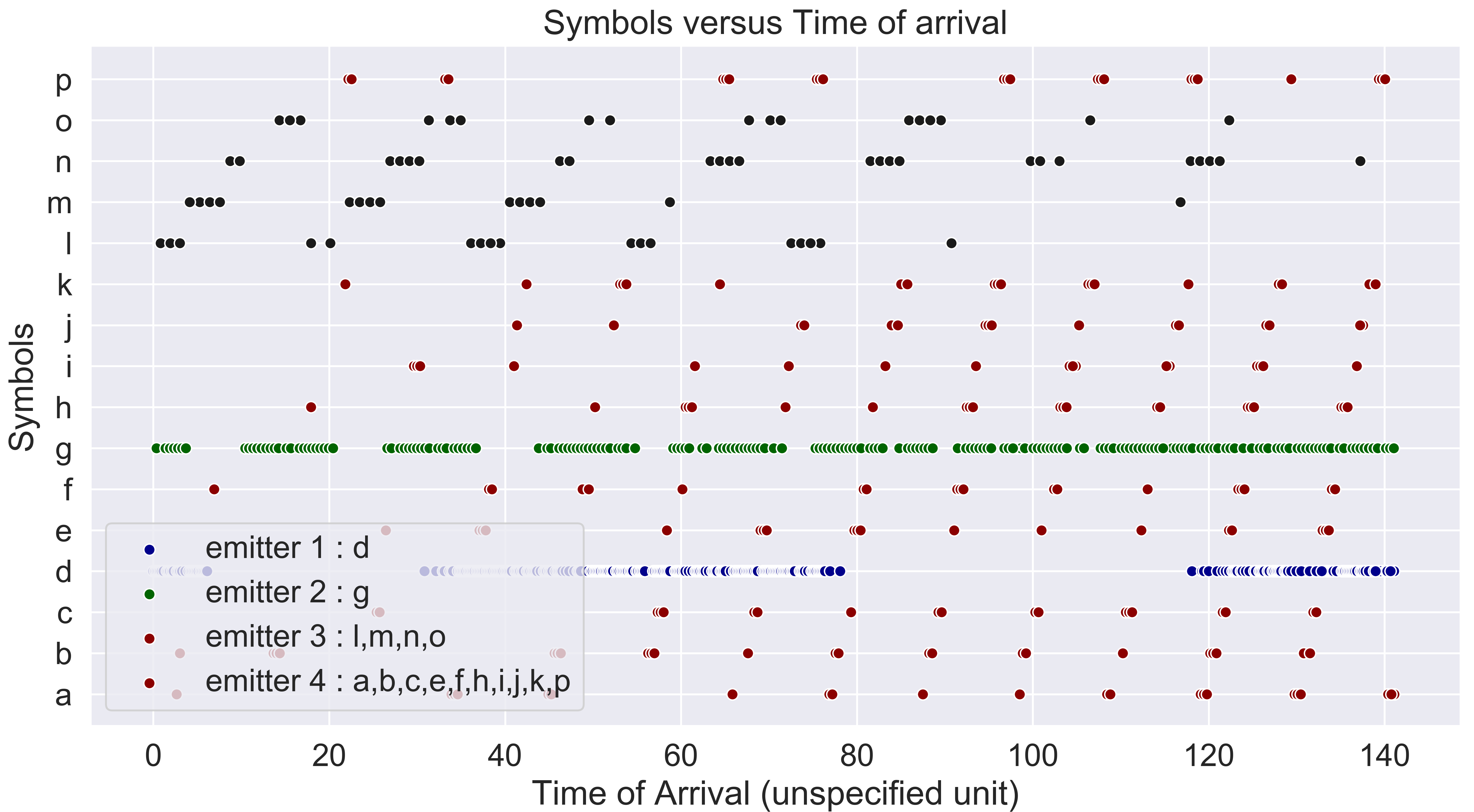}
\caption{For both plots the x-axis is the time of arrival of observed data. The y-axis in the first plot shows the frequency values as they are observed. The second plot shows the result of the preprocessing step where the y-axis represents the symbols when clustering pulses based on their frequency.  Ground truth is given: each color represents one emitter. There are four emitters. Here we can see emitter 1 (in blue) emits only symbol $d$, while emitter 4 (red) emits 10 symbols $a,b,c,e,f,h,i,j,k,p$.\label{exemple_radar}}
\end{figure}

\subsubsection{Hyperparameter calibration}

 We first performed a calibration of the main critical parameters of the different algorithms in order to maximize the average V-measure score on a train set composed of the first 50 instances.
 With this process, the penalization parameter $\gamma$ of \nameOurAlgo\ in the score computation (cf. Equation \eqref{eq:score_global}) is set to 19.  The penalization parameter $\beta$ of \nameAlgoFord\ is set to 0.1.  Three parameters were optimized for \nameSDIF, $\epsilon = 0.05$, threshold $x = 0.9$, and a second threshold needed for the PRI (Pulse Repetition Interval)  transform, $\alpha = 0.5$. 
For the OT optimal transport algorithm, the hierarchical cluster cut is defined as a function of a silhouette score with a threshold set at $s = 0.2$.
These parameters are gathered in Table \ref{table:parameters}.

\subsubsection{Validation of the performances}

With the calibrated parameters, we ran the compared algorithms on the 50 remaining validation instances. Figure \ref{boxplots_radar} displays the distribution of the V-measure scores obtained by the four algorithms on the 50 test scenarios. 

\begin{figure}
\centering
     \includegraphics[width=0.8\textwidth]{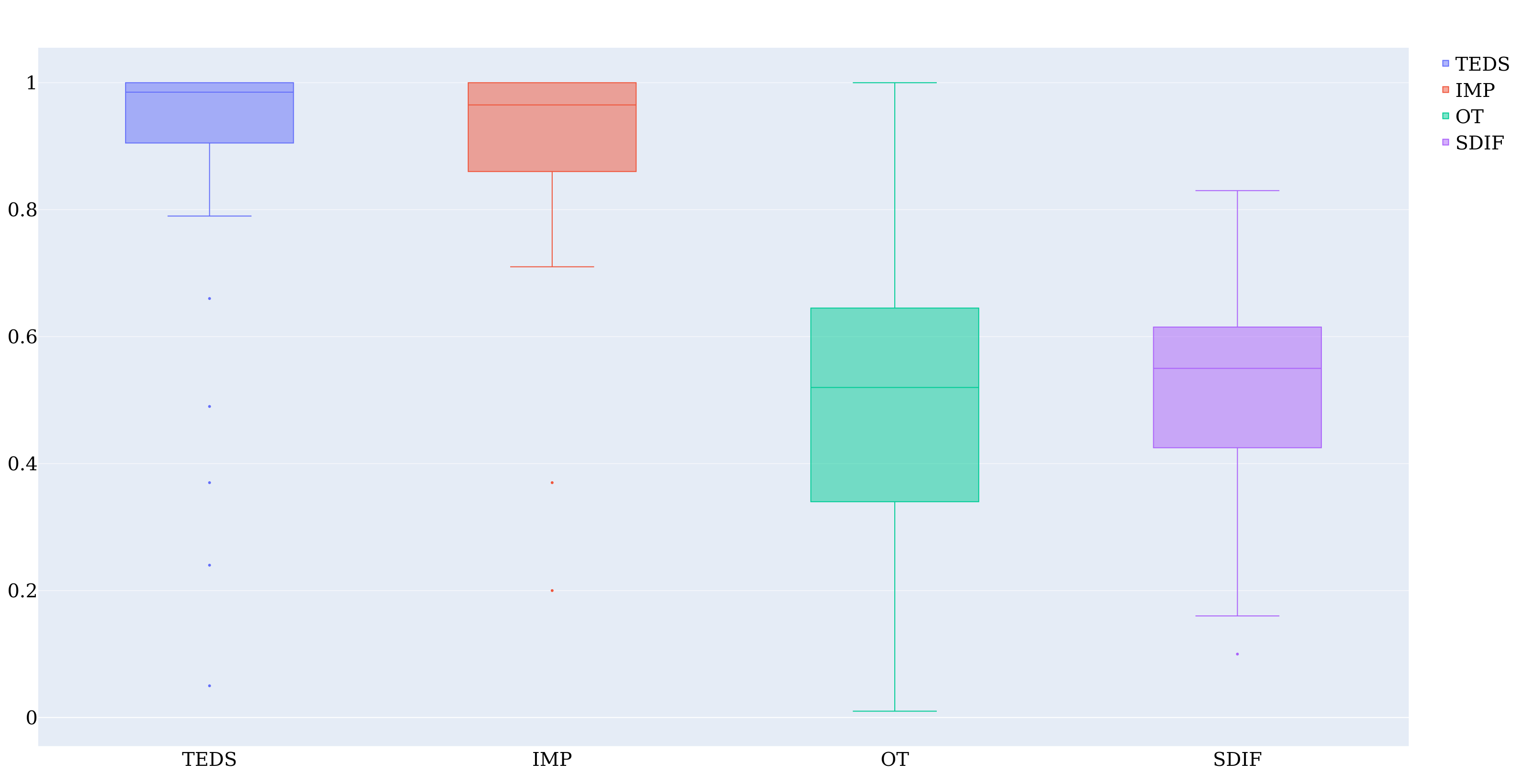}
     \caption{Comparison of deinterleaving algorithm on simulated PDWs. One box represents the distribution of the V-measure scores of an algorithm on the 50 test scenarios.  \label{boxplots_radar}}
 \end{figure}

 \nameOurAlgo\ and \nameAlgoFord\ obtain the best results, with a slightly better median obtained by \nameOurAlgo\ compared to \nameAlgoFord, even if  a  statistical test indicates that the difference in scores between \nameOurAlgo\ and \nameAlgoFord\ is not significant. \nameOurAlgo\ returned 17 perfectly deinterleaved scenarios over 50 against 16 for \nameAlgoFord. These two algorithms do not retrieve the ground truth for the same scenarios, which highlights to some extend the complementary nature of these two approaches for this type of data. Note that \nameOT\ and \nameSDIF\ make less restrictive hypothesis than we made for the data selection i.e., $(\mathcal{P)}$, $(\mathcal{Q})$ and no stop of emission. Therefore, their results do not fully reflect their true efficiency.

\section{Discussion and Perspectives \label{sec:conclusion}}

The main contribution of this paper is to propose a new deinterleaving method for interleaved pulse trains, which exploits both negative log-likelihood minimization of the symbol sequence and negative log-likelihood minimization of their sojourn times distribution.
The choices made in constructing the model (mixture of independent renewal processes, entropy evaluation and penalty term) are supported by theoretical analysis.  We have shown that minimizing the proposed score allows us to recover the true partition associated with the generation of the observed sequence in the large sample limit. This theoretical result is confirmed by an experimental analysis of the score consistency. 
A comparison with other algorithms on synthetic data and electronic warfare data from a realistic simulator confirmed that \nameOurAlgo\ is competitive with state-of-the-art algorithms for the deinterleaving task and scales favorably, in terms of computing time required, to the number of transmitters in the generative process. 

This work opens up various perspectives for future research.  First of all, we have observed that some of the assumptions made in this paper can be violated in reality. In particular, we assumed  in this work that the sub-alphabets are disjoint, which is not always the case in realistic scenarios, because  different transmitters may send pulse with the same frequency. Secondly, we have assumed that all radars in the environment constantly emit, but it may happen that a radar emits no symbol for an extended period of time. The model could take this into account by introducing, for example, a temporary "off" state for each radar, during which it emits no pulses. 

\bibliography{main_arxiv}
\bibliographystyle{plain}

\newpage

\appendix

This appendix supplement presents the proofs of the theorems and propositions  as well as the local search procedure (TabuAP) and the crossover operator (GLPX) used Algorithm 1.

\section{Proof of Theorem 1 \label{app:proof_theorem1}}

\begin{proof}
Let $\Pi'$ be a partition such that  $P = G_{\Pi'}(\{P^e\}_{e \in E(\Pi')})$.  
Let us assume that $\Pi \neq \Pi'$ and let us show that we arrive at a contradiction. 

\textit{First case.} We claim that there exist a sub-alphabet $A_{e_{\Pi'}} \in \Pi'$, such that $|A_{e_{\Pi'}}| \geq 2$ and $A_{e_{\Pi'}} \notin \Pi$. Here $e_{\Pi'} \in E(\Pi')$ denotes an emitter generating symbols in the sub-alphabet $A_{e_{\Pi'}}$.

$A_{e_{\Pi'}}$ contains at least two symbols $a$ and $b$, such that $a$ is in a sub-alphabet $A_{e_{\Pi}}\in \Pi$ and $b \notin A_{e_\Pi}$. Thus, there exists $A_{f_\Pi}\in \Pi$, with $f_\Pi \in E(\Pi)$ and $f_\Pi \neq e_\Pi$, such   that $b \in  A_{f_\Pi}$ with $A_{f_\Pi} \neq A_{e_\Pi}$.

We denote by $k_{min}$ the least possible delay between  two symbols $a$ that $P$ can generate (i.e., $k_{min}= \underset{k \in K^a}{min}(k)$). 
Since $P = G_{\Pi'}(\{P^e\}_{e \in E(\Pi')})$, $b$ cannot be emitted after $a$ in a time delay lower than $k_{min}$. 
Since $P = G_{\Pi}(\{P^e\}_{e \in E(\Pi)})$, $P$ can generate a sequence $z[A_{e_\Pi} \cup A_{f_\Pi} ] = z_{t_a}..z_{t_b}..z_{t_{a'}}$, with $z_{t_a}=a$, $z_{t_b}=b$, $z_{t_{a'}}=a$, where $t_{a'}-t_{a}=k_{\min}$ and $t_a < t_b \leq t_{a'}$,  under assumptions $(\mathcal{P})$  and $(\mathcal{Q})$. This is always  true since $P^{e_\Pi}$ can generate any sub-sequence corresponding to an arbitrary number $\lambda$ of repetitions of symbol $a$, with time delay $k_{min}$ separating each occurrence. For a sufficiently high  $\lambda$, $f_\Pi$ can emit symbol $b$ during the same interval, whatever $K_{e_\Pi}$ and $K_{f_\Pi}$. 

Thus, either $t_b=t_{a'}$, which is impossible for $\Pi'$ since arrival times of symbols from  a single transmitter must be strictly increasing (by definition of a renewal process as given in Section II-A), or it contradicts the previous assertion that states that $t_b - t_a \geq k_{min}$.

 \textit{Second case.} We claim that there exists a sub-alphabet $A_{e_{\Pi}} \in \Pi$, such that $|A_{e_{\Pi}}| \geq 2$ and $A_{e_{\Pi}} \notin \Pi'$.  $A_{e_{\Pi}}$ contains at least two symbols $a$ and $b$, such that $a$ is in a sub-alphabet $A_{e_{\Pi'}} \in \Pi'$ and $b \notin A_{e_{\Pi'}}$. Thus, there exists $A_{f_{\Pi'}} \in \Pi'$, with $f_{\Pi'} \in E(\Pi')$ and $e_{\Pi'} \neq f_{\Pi'}$, such   that, $b \in A_{f_{\Pi'}}$ with $A_{f_{\Pi'}} \neq A_{e_{\Pi'}}$. Using the same reasoning as in the first case, we arrive at a contradiction, which concludes the proof.

\end{proof}

\section{Proof of Proposition 1}

\begin{proof}
This proof is very similar to the proof of Proposition 4.1 in \cite{barbu2009semi}. First we have
\begin{equation}
\label{fikegal1}
\forall\ i\ \in A_e\ \sum_{k \in K^i} q^{e,\sim}_i(k) = 1
\end{equation}
Let $(\lambda_{i})_{i \in z[A_e]}$ be real coefficients. According to (3) and \eqref{fikegal1}, the approached log-likelihood can be written in the form 

\begin{equation}
\label{eq:new_log_approach_likelihood}
\text{log}(\mathcal{L}^T_{\sim}(z^e,t^e)) \nonumber
= \sum_{i,j \in A_e} N^e_{i,j}(T) \text{log} \ p^{e,\sim}_{i,j}  + \sum_{i \in A_e}  \left( \sum_{k \in K^i} N^e_i(k,T) \ \text{log} \ q^{e,\sim}_i(k) + \lambda_i (1 - \sum_{l \in K^i} q^{e,\sim}_i(l)) \right)
\end{equation}

When deriving  \eqref{eq:new_log_approach_likelihood} with respect to $q^{e,\sim}_i(k)$, we observe that a maximum is obtained in $\hat{q}^{e,\sim}_i(k)  = \frac{N^{e}_{i}(k,T)}{\lambda_i}$.

By using (\ref{fikegal1}), we obtain
\begin{align}
1 = \sum_{k \in K^i} q_i^{e,\sim}(k) = \sum_{k \in K_i} \frac{N^{e}_{i}(k,T)}{\lambda_{i}}  = \frac{N^{e}_{i}(T)}{\lambda_{i}}
\end{align}

Thus, the value of $q^{e,\sim}_i(k)$ maximizing  \eqref{eq:new_log_approach_likelihood} is $\hat{q}^{e,\sim}_i(k)   = \frac{N^{e}_{i}(k,T)}{N^{e}_{i}(T)}$.

With the same method, we derive that the value of $p^{e,\sim}_{i,j}$ maximizing  \eqref{eq:new_log_approach_likelihood} is  $\hat{p}^{e,\sim}_{i,j}(T) = \frac{N^e_{i,j}(T)}{N^e_i(T)}$.

\end{proof}

\section{Proof of Theorem 2}

In order to prove Theorem 2, we first show that the generative model $G_{\Pi}$ defined as a  set of Markov renewal chains can be represented as an ergodic finite-state-machine (FSM) source. 

Using the notation proposed in \cite{seroussi2012deinterleaving},  an  FSM  over an alphabet $A$ is defined by the triplet $\mathcal{F} = (\mathcal{S},s_0,f)$ where $\mathcal{S}$ is a finite set of states, $s_0 \in \mathcal{S}$ is a fixed initial state and $f : \mathcal{S} * A \rightarrow{} \mathcal{S}$ is the \textit{next-state function}. The definition of an FSM source is completed from an FSM with the addition of a conditional probability distribution $P_{\mathcal{F}}(.|s)$ with each state $s$ of $\mathcal{S}$, and a probability distribution $P_{\mathcal{F}}^{init}(.)$ on the initial state $s_0$. The FSM source generates a sequence of states $s$ by choosing an initial state $s_0$ according to $P_F^{init}(.)$, then by choosing an action $a_t$ according to $P_F(.|s_{t-1})$, and transitions to the state $s_t=f(s_{t-1},a_t)$.

\subsection{Markov renewal chain as an FSM}

We first observe that a Markov renewal chain $(Z^e,T^e)$ can be represented as an FSM source. For a transmitter $e$ generating a sequence of symbols in $A_e$, we define the state space $\mathcal{S}^e \in A_e \times D^e$, with $D^e$ the set of integers corresponding to the time since a letter of $A_e$ has been transmitted.  $D^e = \{0, 1, \dots, \underset{a\in A_e, k \in K^a }{max} k-1\}$  is a finite set of integers.  

A state $s \in \mathcal{S}^e$ corresponds to a pair  $s =(i,d)$, with $i \in A_e$, the last symbol transmitted by  the emittter $e$ since a duration of $d \in \mathbb{N}$  units of time.

We introduce also the set $A^*_e = A_e \cup *$, with $*$ a new symbol, indicating that the previous symbol $i \in A_e$ lasts one unit of time longer but that there is no transition toward a new letter.

For a Markov renewal chain $(Z^e,T^e)$, we consider the FSM $\mathcal{F}^e = (\mathcal{S}^e, s^e_0, f^e)$, with finite state space $\mathcal{S}^e \in A_e \times D^e$, initial state $s_0$ and next-state function $f^e$.

For the definition of the next-state function $f^e$, given a state $s = (i,d)$ and $a \in A^*_e$, we have 

\begin{equation*}
f^e(s, a) = \begin{cases}
\text{$(i,d +1)$} & \text{if $a = *$}\\
\text{$(a,0)$} & \text{if $a \in A_e$} 
\end{cases}
\end{equation*}

Note that, following a  process different from $P^e$ on states $S^e$ (which we consider in the following for the proof of Theorem 2), it is possible to get actions at state $s$ that should be unavailable from that state. This is the case for instances when $a=*$ while $d\geq k_{max}$, with $k_{max}>\max\limits_{b\in A_e, k \in K^b } k$ an arbitrarily delay limit set for the transmitters (allowing to consider finite state sets in the following). 
In that cases, we consider that $f^e(s,a)=s$ to stay inside $S^e$.  

To complete the definition of an FSM source, for each state $s = (i,d)$ (with $i \in A_e$ and $d \in \mathbb{N}$) and symbol $a \in A^*_e$ we define the conditional distribution as

\begin{equation*}
P_{\mathcal{F}^e}(a | (i,d)) = \begin{cases}
\text{$1 - h_i(d+1)$} & \text{if $a = *$}\\
\text{$p^e_{i,j}h_i(d+1)$} & \text{if $a = j \in A_e$},
\end{cases}
\end{equation*}

\noindent with $h_i(t) = \frac{q^e_i(t)}{\sum\limits_{k \in K^i, k \geq t} q^{e}_i(k)}$, the instantaneous risk of occurrence of a transition at time $t$ for the letter $i$. 

By abuse of notation and for simplicity,  we denote  $P_{\mathcal{F}^e}(s_{t+1}| s_t) = P_{\mathcal{F}^e}(a_{t+1}| s_t)$ the probability of choosing an action $a_{t+1}$ given $s_t$, which leads to the deterministic transition to the state $s_{t+1} = f(s_{t},a_t)$.

\begin{proposition} 
\label{prop_ergodic}
Under assumptions $(\mathcal{Q})$, $(\mathcal{K})$ and $(\mathcal{P})$, each FSM $\mathcal{F}^e = (\mathcal{S}^e, s^e_0, f^e)$ for $e = 1, \dots, E$, is ergodic.
\end{proposition}

\begin{proof}
Let $K_e = \cup_{z \in A_e} K^z $, the set of all sojourn times for transmitter $e$. Let $K^z(d)$ be the set of sojourn times for letter $z$ greater or equal to $d$: $K^z(d) = \{k \in K^z, k \geq d \}$.

First, we observe that all the states $s \in S^e$ of the FSM $\mathcal{F}^e = (\mathcal{S}^e, s^e_0, f^e)$ built from the Markov renewal chain $(Z^e,T^e)$  are recurrent.  Indeed, according to assumptions $(\mathcal{Q})$ and $(\mathcal{P})$, for $i,j \in A^e$, $l \in \mathbb{N}$  we have $P^e(z^e_{l+1} = j | z^e_{l} = i)  > 0$, and for $i \in A^e$, $k\in K^i$, $P^e(x_{l+1} = k | z^e_{l} = i) > 0$. Therefore,  for each state $s = (i,d) \in \mathcal{S}^e$, there is always a positive probability to return to this same state  $(i,d)$ in a finite number of time steps. 

Given a state $(i,d)$, the number of time steps allowing to return to this same state $(i,d)$  with positive probability, is $\Delta T = k^{i}_d + \sum_{j = 1, \dots, |K_e|} \lambda_j k_j^e$, with $k^{i}_d \in K^{i}(d)$, $k_j^e \in K_e$ and the $\lambda_j$'s are non-negative integer coefficients. Given a state $(i,d) \in \mathcal{S}^e$, we denote $T^e(i,d)$ this set of transition  times $\Delta T$.

Given a state $(i,d)$, let us show that the greatest common divisor (gcd) of the set $T^e(i,d)$ is equal to 1.

Let us assume that $gcd(T^e(i,d)) =r $, with $r> 1$ and let us show that we arrive at a contradiction. As $gcd(T^e(i,d)) = r$, 
\begin{equation}
\label{eq:divisor}
\forall (\lambda_1, \dots, \lambda_{K_e}) \in \mathbb{N}^{|K_e|}, \exists \mu \in  \mathbb{N}, k^{i}_d + \sum_{j = 1, \dots, |K_e|} \lambda_j k_j^e = \mu r.
\end{equation}

In particular, when all $\lambda_j$'s are equal to 0,  for $k^{i}_d \in K^{i}(d)$, there exists $\mu_0 \in \mathbb{N}$, such that $k^{i}_d = \mu_0 r$. Thus $r$ is a divisor of $k^{i}_d \in K_e$. 
 If $K_e$ contains only $k^{i}_d$, then according to assumption $(\mathcal{K})$, $k^{i}_d=1$, and we directly have $r = 1$ which is a contradiction. Otherwise $|K_e| > 1$. Thus, given any other sojourn time $k_j^e \in K_e $, such that $k_j^e \neq k^{i}_d$, using Equation \eqref{eq:divisor}, we know that there exists $\mu_j > \mu_0$, such that we have also $k^{i}_d + k_j^e = \mu_j r$, and thus we have $k_j^e = (\mu_j - \mu_0) r$. Then $r$ is a divisor of $k_j^e$. 
Thus, $r> 1$ is a divisor of all sojourn times $k_j^e \in K_e$, which is impossible because $gcd(K_e) = 1$, according to assumption $(\mathcal{K})$. 

Therefore, given a state $(i,d) \in \mathcal{S}^e$, we have shown that $gcd(T^e(i,d)) = 1$, thus each state $s \in \mathcal{S}^e$ is  aperiodic. 

Moreover  each pair of states, $s_i, s_j \in \mathcal{S}^e$, communicate with each other, i.e., when starting from state $s_i$ there is always a positive probability to arrive in state $s_j$ in a finite number of time steps. Therefore, $\mathcal{F}^e$ is ergodic.
\end{proof}

We now show the likelihood equivalence of both representations (Markov renewal chain and FSM source) for a single transmitter. Let $(z^e,t^e) = (z^e_l, t^e_l)_{l = 0, \dots, n^e-1}$ be an observed sequence of $n^e$ symbols  $z_l^e \in A_e$ with their corresponding time of arrival $t^e_l$, generated by $P^e$ until time $T$, given $(z_{-1}, t_{-1})$ the last event before the observation. Let $s^{e} = (s^e_t)_{t=0, \dots, T}$ be the observed sequence of $T+1$ pairs $s^e_t \in \mathcal{S}^e$,
such that for all $l = -1, 0, \dots, n^e-1, \forall t \in \llbracket t^e_l, t^e_{l+1} \llbracket,  s^e_t = (z^e_l, t - t^e_l)$. If $(z_{-1}^e, t_{-1}^e)$ is the last state before the observation, the corresponding last unobserved FSM state is $s^e_{-1} = (z^e_{-1},|t^e_{-1}|-1)$.

\begin{proposition} \label{prop:likelihood_fsm}

Given state $s^e_{-1}$ at time $t=-1$, the conditional likelihood $L^T_{-,\mathcal{F}^e}(s^{e})$ for the sequence $s^{e}=(s^e_i)_{i=0}^T$, generated by the FSM  $\mathcal{F}^e$,  is equal to the conditional  likelihood $\mathcal{L}_{-}^T(z^e,t^e)$ of the corresponding sequence $(z^e_i,t^e_i)_{i=0}^{N^e(T)-1}$ given ($z^e_{-1}, t^e_{-1})$. 

\end{proposition} 

\begin{proof}

We have

\begin{align} \label{eq:likelihood}
L^T_{-,\mathcal{F}^e}(s^{e})
&= \prod_{t =  -1}^{T-1} P_{\mathcal{F}^e}(s^e_{t+1}|s^e_{t}) \\
&= \prod_{t =  -1}^{t^e_0 - 1} P_{\mathcal{F}^e}(s^e_{t+1}|s^e_{t}) \prod_{l = 0}^{n^e-2} \prod_{t =  t^e_l}^{t^e_{l+1} - 1} P_{\mathcal{F}^e}(s^e_{t+1}|s^e_{t}) \prod_{t =   t^e_{n^e}-1}^{T - 1} P_{\mathcal{F}^e}(s^e_{t+1}|s^e_{t}).
\end{align}

We compute the first term  
$\prod_{t =  -1}^{t^e_0 - 1} P_{\mathcal{F}^e}(s^e_{t+1}|s^e_{t})$.
We have, if $t_0 > 0$

\begin{align}
\label{eq:init_fsm}
\prod_{t =  -1}^{t^e_0 - 1} P_{\mathcal{F}^e}(s^e_{t+1}|s^e_{t}) &= \prod_{t=-1}^{t^e_0-2}P_{\mathcal{F}^e}(s^e_{t+1}|s^e_{t}) P_{\mathcal{F}^e}(s^e_{t_0}|s^e_{t_0-1}) \\ \nonumber
& = \prod_{t=-1}^{t^e_0-2} (1- h_{z^e_{-1}}(|t^e_{-1}| + t + 1)) P^e(z^e_0|z^e_{-1})h_{z^e_{-1}}(t^e_0 + |t^e_{-1}|) \\ \nonumber
&= \frac{\sum_{k \geqslant |t_{-1}| + t_0} q_{z^e_{-1}}(k)}{\sum_{k \geqslant |t_{-1}|} q_{z^e_{-1}}(k)} P^e(z^e_0|z^e_{-1}) \frac{q_{z^e_{-1}}(t^e_0 + |t^e_{-1}|)}{\sum_{k \geqslant |t^e_{-1}| + t_0} q_{z^e_{-1}}(k)} \\ \nonumber
&=  P^e(z^e_0|z^e_{-1}) \frac{P(t^e_0 + |t^e_{-1}||z^e_{-1})}{R^e_{-1}(|t^e_{-1}|-1)}
\end{align}

If $t_0=0$, \eqref{eq:init_fsm} is immediate.

Now let us compute separately the terms $\prod_{t =  t^e_l}^{t^e_{l+1} - 1} P_{\mathcal{F}^e}(s^e_{t+1}|s^e_{t})$ and $\prod_{t =   t^e_{n^e-1}}^{T - 1} P_{\mathcal{F}^e}(s^e_{t+1}|s^e_{t})$.

For $l = 0, \dots, n^e-2$, given  states $s^e_{t^e_l} = (z^e_l = i, d=0)$ and $s^e_{t^e_{l+1}} = (z^e_{l+1} = j, d=0)$,
\begin{align}\label{eq:first_term}
\prod_{t =  t^e_l}^{t^e_{l+1} - 1} P_{\mathcal{F}^e}(s^e_{t+1}|s^e_{t}) 
&=   [\prod_{t = 0 }^{t^e_{l+1} - t^e_{l} - 2}  1 - h_i(t+1)] p^e_{i,j}h_i(t^e_{l+1} - t^e_{l}) \\ \nonumber
&=   R_i(t^e_{l+1} - t^e_{l} - 1) p^e_{i,j}h_i(t^e_{l+1} - t^e_{l})
\end{align}

Thus,
\begin{align}\label{eq:first_part}
\prod_{t =  t^e_l}^{t^e_{l+1} - 1} P_{\mathcal{F}^e}(s^e_{t+1}|s^e_{t}) &= p^e_{i,j} \sum_{k \in K^i, k > t^e_{l+1} - t^e_{l} - 1} q^e_i(k) \times \frac{q^e_{i}(t^e_{l+1} - t^e_{l})}{\sum_{k^ \in K^i, k \geq t^e_{l+1} - t^e_{l}} q^e_i(k)} \\ \nonumber
&= p^e_{i,j} q^e_i(t^e_{l+1} - t^e_{l}) \\ \nonumber
& = P^e(z^e_{l+1}  | z^e_{l}) \times 
P^e(t^e_{l+1} - t^e_{l}|z^e_l)
\end{align}

Furthermore, given a state $s^e_{t^e_{n^e-1}} = (z^e_{n^e-1}, d=0)$,

\begin{align}\label{eq:second_part}
\prod_{t =   t^e_{n^e-1}}^{T - 1} P_{\mathcal{F}^e}(s^e_{t+1}|s^e_{t}) 
&= \prod_{t = 0 }^{T - t^e_{n^e-1} - 1} (1 - h_{z^e_{n^e-1}}(t+1)) \\ \nonumber
&= R^e_{n^e-1}(T - t^e_{n^e-1})
\end{align}

Using (2), \eqref{eq:likelihood}, \eqref{eq:init_fsm}, \eqref{eq:first_part} and \eqref{eq:second_part} we finally obtain:

\begin{align}
L^T_{-,\mathcal{F}^e}(s^e)  
=  \mathcal{L}_{-}^T(z^e,t^e).
\end{align}

\end{proof}

\subsection{Collection of Markov renewal chains  as a FSM}

Given a partition $\Pi = \{A_{e_\Pi}\}_{e_\Pi \in E(\Pi)}$ of $\mathcal{A}$ into  $m$  
non-empty and disjoint sub-alphabets, the global generative process $G_{\Pi}$  corresponds to a set of  $m$ independent Markov renewal chains associated to each individual transmitter $e_\Pi \in E(\Pi)$, each emitting from its own sub-alphabet $A_{e}$. 
This process can be represented as a global FSM $\mathcal{F}_{\Pi} = (\mathcal{S}_\Pi, \mathcal{\textbf{s}}_{0}^\Pi, f_\Pi)$, with state set $\mathcal{S}_\Pi = \prod_{e \in E(\Pi)} 
\mathcal{S}^{e}$ 
corresponding to the Cartesian product of the states of the individual FSM of  each transmitter $e \in E(\Pi)$  defined in previous section. 
Given a state $\mathcal{\textbf{s}} \in \mathcal{S}_\Pi$, with $\mathcal{\textbf{s}} = \{s^{e}\}_{e \in E(\Pi)}$,  
and an action $\mathbf{a} \in A^*_\Pi$, with $A^*_\Pi=\prod_{e \in E(\Pi)} A^*_{e}$ and
 $\mathbf{a} =(a^{e})_{e \in E(\Pi)}$, we have $f_\Pi(\mathcal{\textbf{s}}, \mathbf{a}) = (f^{e}(s^{e}, a^{e}))_{e \in E(\Pi)}$, where $f^{e}$ stands for the next state function from the corresponding  individual FSM ${\cal F}^{e}$. 
For each state we define and $P_{\mathcal{F}_\Pi}(. | \textbf{s}) = \prod_{e \in E(\Pi)} P_{\mathcal{F}^{e}}(. | s^{e}) $.

\begin{proposition}
 \label{prop_ergodic_global}
\textit{Under assumptions $(\mathcal{P})$, $(\mathcal{Q})$ and $(\mathcal{K})$, the FSM $\mathcal{F}_{\Pi} = (\mathcal{S}, \mathcal{\textbf{s}}_0, f)$ is ergodic.}
\end{proposition}

\begin{proof}

Under assumptions $(\mathcal{P})$, $(\mathcal{Q})$ and $(\mathcal{K})$, according to Proposition \ref{prop_ergodic} each FSM  $\mathcal{F}^e = (\mathcal{S}^e, s^e_0, f^e)$ for  $e \in E(\Pi)$, 
is ergodic. 

Therefore, for  $e \in E(\Pi)$, there exists some $n^e_0$, such that for every integer $n^e$, greater than $n^e_0$, and for every pair of states $i, j \in \mathcal{S}^e$, there exists a $i$-to-$j$ path of length $n^e_0$, a $i$-to-$j$ path of length $n^e_0 + 1$, a $i$-to-$j$ path of length $n^e_0 + 2$, and so on. We denote $n_0 = \underset{  e \in E(\Pi)}{max} n^e_0$. Thus, given any two states   $\mathcal{\textbf{s}}_i = \{s_i^e\}_{e \in E(\Pi)}$ 
and $\mathcal{\textbf{s}}_j = \{s_i^e\}_{e \in E(\Pi)}$, 
there always exists a path of length $n_0$ allowing to reach  $\mathcal{\textbf{s}}_j$ from $\mathcal{\textbf{s}}_i$. 
Therefore $\mathcal{F}_{\Pi}$ is ergodic. 
\end{proof}

Given a partition $\Pi = \{A_{e}\}_{e \in E(\Pi)}$ of $\mathcal{A}$ into $m$ non-empty and disjoint sub-alphabets, let $(z,t)$ be sequence of $n$ symbols  $z \in \mathcal{A}$ with their corresponding time of arrival $t_l$, generated by $P = G_{\Pi}$ and observed in the time window $ \llbracket 0,T \rrbracket$. In this sequence, each transmitter $e \in E(\Pi)$ generates a sequence $(z^e_l, t^e_l)_{l = 0, \dots, n^e-1}$ of $n^e$ symbols  $z_l^e \in A_e$ with their corresponding time of arrival $t^e_l \in \mathbb{N}$ and their last non observed state $(z^e_{-1},t^e_{-1})$.
Let $\mathbf{s} = (\mathbf{s}_t)_{t=0, \dots, T}$ be the sequence of $T+1$ states   $\mathbf{s}_t=\{s^{e}_t\}_{e \in E(\Pi)} 
\in \mathcal{S}_\Pi$, such that $\forall e \in E(\Pi),  
\ \forall l = -1, 0, \dots, n^{e}-1, \forall t \in \llbracket t^{e}_l, t^{e}_{l+1} \llbracket , s^{e}_t = (z^{e}_l, t - t^{e}_l)$ given $s^e_{-1} = (z^e_{-1},|t^e_{-1}|-1)$.

\begin{proposition}
\label{prop:likelihood_global_fsm}
    Given a state $\mathbf{s_{-1}}$ at time $t=-1$, the conditional likelihood $L^T_{-,\mathcal{F}_\Pi}(\mathbf{s})$ of this sequence $\mathbf{s}$, generated by the FSM  $\mathcal{F}_\Pi$,  and observed until time $T$, 
    is equal to the conditional likelihood $\mathcal{L}^T_{-,\Pi}(z,t)$ of the sequence $(z,t)$ generated by $P$ given $(z^e_{-1},t^e_{-1})_{e \in E(\Pi)}$. 
\end{proposition}

\begin{proof}

This is trivial, as all the transmitters are independent, and using Proposition \ref{prop:likelihood_fsm}.
\end{proof}

Now let $\Pi$ be a partition  associated to 
a generative model $G_{\Pi}$.
The maximum of the exact and approached  likelihood, given by  (2) and (3), of the sequence $(z,t)$ following $(z_{-1},t_{-1})$ are respectively  denoted as $\widehat{\mathcal{L}}_{-, \Pi}^T(z,t)$ and $\widehat{\mathcal{L}}^T_{\sim,\Pi}(z,t)$.

\begin{proposition}
 \label{prop:approx_likelihood_max}
 Given two partitions $\Pi$ and $\Pi'$, for any sequence  $(z,t)$ sampled from $P_{\Pi|\Pi'}$, we have: 
\begin{equation}
\frac{\text{log}(\widehat{\mathcal{L}}_{-,\Pi'}^T(z,t)) - \text{log}(\widehat{\mathcal{L}}^T_{\sim,\Pi'}(z,t))}{T} \rightarrow 0 \ 
\end{equation}
\noindent a.a.s as $T \rightarrow \infty$.

\end{proposition}

\begin{proof}

For any emitter $e \in E(\Pi')$, we denote by $\hat{p}^e$ 
and $\hat{q}^e$ the estimators that maximize   the log-likelihood function:  
\begin{eqnarray}
\label{eq:complete_log_likelihood}
\text{log}(\widehat{\mathcal{L}}_{-}^T(z^e,t^e)) & =&  \text{log}(\hat{p}^e_{z^e_{-1},z^e_0}(T)) + \text{log}(\hat{q}^e_{z^e_{-1}}(t^e_0-t^e_{-1},T)) \nonumber \\ &-& \text{log}(\widehat{R}^e_{-1}(|t^e_{-1}|-1))  \nonumber \\
& + &\sum_{i \in A_e} \sum_{j \in A_e} N^e_{i,j}(T) \text{log} \ \hat{p}^e_{i,j}(T) \nonumber \\
& + &\sum_{i \in A_e} \sum_{k \in K_i}  N^e_i(k,T)\ \text{log} \ \hat{q}^e_{i}(k,T) \nonumber \\
& + &\text{log}(\widehat{R}^e_{N^e(T)-1}(u_T))
\end{eqnarray}
with $u_T=T-t^e_{N^e(T)-1}$.

We first look at each estimator $\hat{q}^e_i$ for any symbol $i$ from $A_e$, that are all independent from  $\hat{p}^e$ in that maximization, and also independent from each other $\hat{q}^e_j$ whenever $j \neq i$. By construction, our algorithm considers all necessary delays as included in $K^{z^e_i}$ for each $i$. 
For a sufficiently large observation window $\llbracket 0,T \rrbracket$, 
we thus have, for each transition starting from symbol $i$, a delay $k \in K^{z^e_i}$ that match the delay of the observed transition, and thus one different parameter $\hat{q}^e_i(k,T)>0$.   
Since the sequence is sampled from a process that guarantees the feasibility for $P_{\Pi'}$, we know that two symbols from $A_e$ cannot be emitted at the same time, which discards $k=0$. 

Recalling $\widehat{R}^e_i(x) = P(t^e_{i+1} - t^e_i > x|z_i)$, we can write : 
\begin{equation}
\widehat{R}^e_{-1}(|t^e_{-1}|-1) = \sum_{k \in K^{z^e_{-1}}, k >  |t^e_{-1}|-1} \hat{q}^e_{z^e_{-1}}(k,T).
\end{equation}
and 
\begin{equation}
\widehat{R}^e_{N^e(T)-1}(u_T) = \sum_{k \in K^{z^e_{N^e(T)-1}}, k > u_T} \hat{q}^e_{z^e_{N^e(T)-1}}(k,T).
\end{equation}
For a given symbol $i$ and a given delay $k \in K^{z^e_i}$, we note that $\hat{q}^e_i(k, T)$ can be part of the expression of $\widehat{R}^e_{-1}(|t^e_{-1}|-1)$ (if $z^e_{-1}=i$ and $k > |t^e_{-1}|$) or $\widehat{R}^e_{N^e(T)-1}(u_T)$ (if $z^e_{N^e(T)-1}=i$ and $k > u_T$).

Using of the same method as in  the proof for  Proposition 1 with these additional factors, we obtain the following equation by canceling the derivative of the Lagrangian expression w.r.t. $\hat{q}^e_i(k, T)$: 
\begin{equation}
\lambda_i=-\xi^e_{-1}(k,i) 
+ \frac{I(k=t^e_0-t^e_{-1}) + N^e_i(k,T)}{\hat{q}^e_i(k, T)} + \xi^e_u(k,i) 
\end{equation}

where I(.) is the indicator function that equals 1 if the argument is true, 0 otherwise, and
\begin{equation}
\nonumber 
\xi^e_{-1}(k,i)=
\begin{cases}
    \frac{1}{\widehat{R}^e_{-1}(|t^e_{-1}|-1)} & \text{if }  z^e_{-1}=i \ \text{and}\  k\geq |t^e_{-1}| \\
    0 & \text{otherwise}, 
\end{cases}
\end{equation}
\begin{equation}
\nonumber 
\xi^e_u(k,i) =
\begin{cases}
    \frac{1}{\widehat{R}^e_{N^e(T)-1}(u_T)} & \text{if }  z^e_{N^e(T)-1}=i \ \text{and} \   k > u_T \\
    0 & \text{otherwise}. 
\end{cases}
\end{equation}

Next, since $\xi^e_{-1}(k,i)$ and $\xi^e_u(k,i)$ are both upper-bounded by $\frac{1}{\hat{q}^e_i(k, T)}$, we observe that $\lambda_i$ is bounded as: 
$$\frac{N^e_i(k,T)+2}{\hat{q}^e_i(k, T)} \geq \lambda_i \geq \frac{N^e_i(k,T)-1}{\hat{q}^e_i(k, T)}.$$
Since $\hat{q}^e_i(k, T) \leq 1$, we know that $\lambda_i \geq N^e_i(k,T)-1$. We also deduce that
$$\frac{N^e_i(k,T)+2}{\lambda_i} \geq  \hat{q}^e_i(k, T) \geq \frac{N^e_i(k,T)-1}{\lambda_i}.$$
Considering that sequences are sampled from $P_{\Pi|\Pi'}$, and since we know that the process is ergodic from assumption $\cal K$, we have for every symbol $i$ and delay in $k \in K^e_i$: 
$$\frac{N^e_i(k,T)}{T} \rightarrow \frac{q^*_{i}(k)}{\delta^*_{i}}  \text{  a.a.s. as } T \rightarrow \infty$$
with $\delta^*_{i}$ the expected delay between two occurences of symbol $i$ in the sequence and $q^*_{i}(k)$ the true  probability of delay $k$ from symbol $i$ in the process $P_{\Pi|\Pi'}$. Thus, since $\delta^*_{i}$ and $q^*_{i}(k)$ are stationary, we get that both $N^e_i(k,T)$ and $\lambda_i$ increase with $T$. We obtain that both bounds of $\hat{q}^e_i(k, T)$ converge in probability to $\frac{N^e_i(k,T)}{\lambda_i}$, thus: $\hat{q}^e_i(k, T) - \frac{N^e_i(k,T)}{\lambda_i} \rightarrow 0$ a.a.s as $T \rightarrow \infty$.

From the Karush–Kuhn–Tucker conditions, we know that:  $\lambda_i (1-\sum_{l \in K^{i}} \hat{q}^e_i(l, T)) =0$, thus $\lambda_i$ converges in probability toward $\sum_{l \in K^{i}} N^e_i(l,T)$ when $T$ goes toward infinity.  
Finally, we get that: $(\hat{q}^e_i(k, T)  - \frac{N^e_i(k,T)}{\sum_{l \in K^{i}} N^e_i(l,T)}) \rightarrow 0$ a.a.s as $T\rightarrow \infty$, 
which corresponds to the estimator obtained from the maximization of the approximation $\widehat{\mathcal{L}}^T_{\sim,\Pi'}(z,t)$. 

Demonstrating the asymptotic equivalence of optima for $\hat{p}^e$ estimators is more direct. We can simply note that both likelihoods only differ from the first transition, that is ignored from $\widehat{\mathcal{L}}^T_{\sim,\Pi'}(z,t)$. The only possible  difference thus can lies on $\hat{p}^e_{z^e_{-1},z^e_0}(T)$, whose estimation from  $\widehat{\mathcal{L}}_{-,\Pi'}^T(z,t)$ gives $\frac{N^e_{z^e_{-1},z^e_0}(T)+1}{N^e_{z^e_{-1}}(T)+1}$, while we obtain $\frac{N^e_{z^e_{-1},z^e_0}(T)}{N^e_{z^e_{-1}}(T)}$ from the maximization of  $\widehat{\mathcal{L}}^T_{\sim,\Pi'}(z,t)$. Similarly as above, we can simply note that, for any $i \in A^e$ and $j \in A^e$:
$$\frac{N^e_i(T)}{T} \rightarrow \frac{1}{\delta^*_{i}}\ \text{a.a.s. as}\ T \rightarrow \infty$$

and: 
$$\frac{N^e_{i,j}(T)}{T} \rightarrow \frac{p^*_{i,j}}{\delta^*_{i}}  \text{  a.a.s. as } T \rightarrow \infty$$
with $p^*_{i,j}$ the stationary probability for emitting $j$ from $i$ in process $P_{\Pi|\Pi'}$. Thus, as all components except $T$ are constants, we get, recalling $\hat{p}^e_{z^e_{-1},z^e_0}(T) = \frac{N^e_{z^e_{-1},z^e_0}(T)+1}{N^e_{z^e_{-1}}(T)+1}$ :   $$\frac{N^e_{z^e_{-1},z^e_0}(T)+1}{N^e_{z^e_{-1}}(T)+1} - \frac{N^e_{z^e_{-1},z^e_0}(T)}{N^e_{z^e_{-1}}(T)} \rightarrow 0 \text{  a.a.s. as } T \rightarrow \infty$$

Thus, we know that all estimators from both optimization converge asymptotically toward the same optima. We directly obtain that:
\begin{equation}
\frac{\text{log}(\widehat{\mathcal{L}}_{-,\Pi'}^T(z,t)) - \text{log}(\widehat{\mathcal{L}}^T_{\sim,\Pi'}(z,t))}{T} \rightarrow 0 \end{equation}
\noindent a.a.s. as $T \rightarrow \infty$, since every component from $\widehat{\mathcal{L}}_{-,\Pi'}^T(z,t)$ not present in $\widehat{\mathcal{L}}^T_{\sim,\Pi'}(z,t)$ converge to constant values depending on stationary true parameters of the process $P_{\Pi|\Pi'}$.

\end{proof}

Now using Theorem 1 and 
Propositions  \ref{prop_ergodic_global}, \ref{prop:likelihood_global_fsm} and \ref{prop:approx_likelihood_max} we prove Theorem 2.

\begin{proof} 
For this proof we took inspiration from the proof of  Lemma 10 in \cite{seroussi2012deinterleaving} and we use the representation of the generative model $G_{\Pi}$ as a global ergodic FSM ${\cal F}_\Pi$, as  described above. 
Let $\Pi^{'}$ be a partition of $\mathcal{A}$ such that $\Pi^{'} \neq \Pi$. 

Let respectively $\mathcal{F}_{\Pi}=(S_\Pi, s^\Pi_0, f_\Pi)$ and $\mathcal{F}_{\Pi'}=(S_{\Pi'}, s^{\Pi'}_0, f_{\Pi'})$ be the FSM representing  $P_{\Pi} = G_{\Pi}(\{P^e\}_{e \in E(\Pi)})$ and  $P_{\Pi'} = G_{\Pi'}(\{P^e\}_{e \in E(\Pi')})$. Let $P_{\mathcal{F}_{\Pi}}$ and $P_{\mathcal{F}_{\Pi'}}$ be the respective probability distribution defined on states of the respective FSM $\mathcal{F}_{\Pi}$ and  $\mathcal{F}_{\Pi'}$ to represent the processes. 

First of all, let  $\mathcal{F}^+=(S^+, s^+_0, f^+)$ be a common refinement of $\mathcal{F}_{\Pi}$ and $\mathcal{F}_{\Pi'}$, such that there exist functions $g_\Pi: S^{+} \rightarrow S_\Pi$ and $g_{\Pi'}: S^{+} \rightarrow S_{\Pi'}$ allowing to recover states of both processes from $S^+$. That is,  for any sequence $(z,t)$ observed until time $T$ and its corresponding state sequence $(s^+_t)_{t=0}^T$, the respective state sequences $\left(s_t^\Pi\right)_{t=0}^T$ and $\left(s_t^{\Pi'}\right)_{t=0}^T$ satisfy $s_t^\Pi=g_\Pi\left(s_t^{+}\right)$ and $s_t^{\Pi'}=g_{\Pi'}\left(s_t^{+}\right)$, for any $t = 0 \dots T$.  
According to \cite{seroussi2012deinterleaving}, it is always possible to construct a common refinement of two FSMs,  whose state set $S^+$ is the Cartesian product of the two respective state sets.  
Thus, for any state $s_t^+ \in S^+$, we consider $s_t^+=(s_t^\Pi,s_t^{\Pi'})$, with $s_t^\Pi \in {\cal S}^\Pi$ and $s_t^{\Pi'} \in {\cal S}^{\Pi'}$. 

At any time step $t$, $s_t^\Pi \in {\cal S}_\Pi$ (resp. $s_t^{\Pi'} \in {\cal S}_{\Pi'}$) can be recovered from $s_t^+$, by considering the function $g_\Pi$ (resp. $g_{\Pi'}$) that selects the corresponding part to $\Pi$ (resp. $\Pi'$) from $s_t^+$.  

In the following, we note $P_{{\cal F}}^+$ the adaptation of the probabilities  $P_{\cal F}$ to the states of the FSM ${\cal F}^+$, that is: $\forall s_i^+ \in {{\cal S}^+}, \forall a \in A^*_\Pi, P^+_{{\cal F}_\Pi}(a|s_i^+)= P_{{\cal F}_\Pi}(a|g_\Pi(s_i^+))$. Thus,  $\forall (s_i^+,s_j^+) \in {{\cal S}^+}\times{{\cal S}^+}, P^+_{{\cal F}_\Pi}(s_j^+|s_i^+)=\sum_{a \in A^*_\Pi} P_{{\cal F}_\Pi}(a|g_\Pi(s_i^+)) I(f^+(s_i^+,a)=s_j^+)$, with $I(.)$ the indicator  function. 
One notes that to make $P^+_{{\cal F}_\Pi}$ irreducible, we consider in the following a set $S^+$ only containing states that are reachable from the process $P_\Pi$. 
Thus, ${\cal S}^+$ is a finite state space, since every transmitter owns a finite maximal emission delay. Also, as $P_{\mathcal{F}_{\Pi}}$ is ergodic according to Proposition \ref{prop_ergodic} under assumptions $(\mathcal{P})$, $(\mathcal{Q})$ and $(\mathcal{K})$,  and $P^+_{{\cal F}_\Pi}$ is ergodic under the same assumptions.

In that setting, the asymptotic normalized Kullbak-Liebler divergence 
relative to $P_\Pi$ from $P_{\Pi'}$ can be written as: 

\begin{eqnarray}
\label{Pins}
D(P_{\Pi}||P_{\Pi'}) & = & 
D(P^+_{\mathcal{F}_{\Pi}}||P^+_{\mathcal{F}_{\Pi'}})  \\
& = & \sum_{s \in S^+} P^+_{\mathcal{F}_{\Pi}}(s)D(P^+_{\mathcal{F}_{\Pi}}(.|s) || P^+_{\mathcal{F}_{\Pi'}}(.|s)), \nonumber
\end{eqnarray}

\noindent with $D(P^+_{\mathcal{F}_{\Pi}}(.|s) || P^+_{\mathcal{F}_{\Pi'}}(.|s))$ the Kullback-Leibler divergence of the  conditional distribution   $P^+_{\mathcal{F}_{\Pi}}(.|s)$ from $P^+_{\mathcal{F}_{\Pi}'}(.|s)$. This provides us a powerful tool to compare the processes.

However, as already mentioned above, impossible actions for $P_{\Pi'}$ can be emitted from $P_\Pi$ when for instance an emission delay is lower (resp. greater) than the minimal (resp. maximal) delay for the corresponding transmitters from $\Pi'$. 

\begin{itemize} 
\item First, regarding maximal delay, when following $P_\Pi$ it may happen that a symbol $z$ from a transmitter $e$ of $\Pi$ has not been observed since a long time $d$, because for instance other symbols from $e$ were repeated many times after its last occurrence. 

\item On the other hand, actions can be blocked because they correspond to the simultaneous emission of symbols that shouldn't be observed at the same timestep, due to the fact they belong to the same sub-alphabet of $\Pi'$.
\end{itemize}

Such cases result in   an undefined KL divergence, with some impossible transitions or incompatible states for $\Pi'$. Considering sufficiently long time windows, such bad events occur with probability 1 for every $\Pi' \neq \Pi$, for sequences generated from $P_\Pi$, due to the incompatibility result stated in theorem 1.  As explained in section III-B, such situation gets $C^T_{\Pi'}(z,t):=+\infty$ when detected, which allows our algorithm to discard the corresponding partition. This respects the positive difference from the theorem, as $C^T_{\Pi}(z,t)$ always remain defined after a sufficiently large $T$. 

However, such bad events can occur very rarely for some scattered sets of emission delays. In the following, we show that even for sequences feasible for both process $P_\Pi$ and $P_\Pi'$, our algorithm is able to correctly identify the best partition, after a sufficiently long period of time $T$. 
For this, we consider sequences sampled from $P_{\Pi|\Pi'}$, which corresponds to a process  restricting $P_{\Pi}$ to the subset of possible sequences for $P_{\Pi'}$.

Thus, we consider an alternative FSM ${\cal F}^*=(S^*, s_0^*, f^*)$, which corresponds to the common refinement ${\cal F}^+$ in which we remove from $f^+$ all  transitions which are either impossible regarding $\Pi$ or $\Pi'$, or that lead to any $s \in {\cal S}^+$ from which every sequence to come back in $s$ implies impossible actions  regarding $\Pi$ or $\Pi'$.  ${\cal S}^* \subset {\cal S}^+$ contains only states that are reachable following possible transitions for both processes, starting from any state from ${\cal S}^*$. 
 By definition, $s \in {\cal S}^*$ implies there exists a set 
 of trajectories in ${\cal F}^*$ that allows us to leave and go back to $s$. 
In that FSM, following process $P_{\Pi|\Pi'}$ corresponds to 
considering $P^*_{{\cal F}_\Pi}$ as the probabilities from  $P^+_{{\cal F}_\Pi}$ adapted to ${\cal F}^*$, where the probability mass of actions leading deterministically outside ${\cal S}^*$ from any state in ${\cal S}^*$ are evenly redistributed on other admissible actions regarding $f^*$. 
Using this, we have: 
\begin{itemize}
\item 
By definition,  starting from a state  from ${\cal S}^*$, there always exists an alternative to the simultaneous emission of symbols from the same transmitter in $E(\Pi')$. These symbols can be emitted from their corresponding transmitters in $E(\Pi)$ at different timesteps. When such a delayed  emission occurs, the corresponding transmitter recovers its whole set of possible delays, with gcd of 1, hence ensuring aperiodicity of the state in ${\cal F}^*$.
   \item 
    By construction, as ${\cal S}^*$ is a finite set,  there always exists a finite $k_{max}$ sufficiently high to  maintain aperiodicity of any state from ${\cal S}^*$ despite 
    blocked transitions due to maximal delays.
\end{itemize}
Thus, if ${\cal S}^*$ is non-empty (some sequences can be generated by both processes when starting from $s_0^\Pi$),  $P^*_{{\cal F}_\Pi}$ is ergodic under the same assumptions as above for the ergodicity of $P_{{\cal F}_\Pi}$ on ${\cal S}$.
In that setting, the KL divergence relative to $P_{\Pi|\Pi'}$ from $P_{\Pi'}$ can be given as 
\begin{eqnarray}
\label{Pins}
D(P_{\Pi|\Pi'}||P_{\Pi'}) & = & 
D(P^*_{\mathcal{F}_{\Pi}}||P^*_{\mathcal{F}_{\Pi'}}) \\ & =  & D(P^*_{\mathcal{F}_{\Pi}}||P^+_{\mathcal{F}_{\Pi'}})  \\
& = & \sum_{s \in S^*} P^*_{\mathcal{F}_{\Pi}}(s)D(P^*_{\mathcal{F}_{\Pi}}(.|s) || P^*_{\mathcal{F}_{\Pi'}}(.|s)) \nonumber
\end{eqnarray}

Let $(z,t)$ be sequence of $n$ symbols  $z_l \in \mathcal{A}$ with their corresponding time of arrival $t_l \in \mathbb{N}$, generated by $P_{\Pi|\Pi'}$ 
and observed in the time window $ \llbracket 0,T \rrbracket$. In this sequence, each transmitter $e$, generates a sequence $(z^e_l, t^e_l)_{l = 0, \dots, n^e-1}$ of $n^e$ symbols  $z^e \in A_e$ with their corresponding time of arrival $t^e_l \in \mathbb{N}$. Let $\mathbf{s}^T = (\mathbf{s}_t)_{t=0, \dots, T}$ be the sequence of $T+1$  vector states $\mathbf{s}_t \in S^*$, corresponding to the generation of $(z,t)$ with  $P^*_{{\cal F}_{\Pi}}$ on ${\cal F}^*$.

Let $\hat{P}^*_{\mathcal{F}_\Pi}(\mathbf{s}^T)$ be the probability estimator that maximizes  the log-likelihood  $\log(\hat{L}^{T,*}_{\mathcal{F}_\Pi}(\mathbf{s}^T))))$, 
with corresponding empirical entropy $H^{T,*}_{\mathcal{F}_\Pi}(\mathbf{s}^T)  = - \log(\hat{L}^{T,*}_{\mathcal{F}_\Pi}(\mathbf{s}^T))))$. We have  

\begin{equation}
 H^{T,*}_{\mathcal{F}_\Pi}(\mathbf{s}^T)
   = - \sum_{s,s' \in \mathbf{s}^T} N_{s,s'}(T) \text{log} \ \hat{P}^*_{\mathcal{F}_\Pi}(s'|s),
\end{equation}
\noindent with $N_{s,s'}(T)$ the number of transitions from state $s \in S^*$ to state $s' \in S^*$ observed in the sequence $\mathbf{s}^T$ until time $T$. We have: 
\begin{equation}
\frac{1}{T} (H^{T,*}_{\mathcal{F}_{\Pi'}}(\mathbf{s}^T) - H^{T,*}_{\mathcal{F}_\Pi}(\mathbf{s}^T)) 
 =  \sum_{s \in \mathbf{s}^T} N_{s}/T \sum_{s' \in \mathbf{s}^T} [N_{s,s'}/N_{s}] \log [\hat{P}^*_{\mathcal{F}_\Pi}(s'|s)/ \hat{P}^*_{\mathcal{F}_{\Pi'}}(s'|s)]
\end{equation}

As $\hat{P}^*_{\mathcal{F}_\Pi}(\mathbf{s}^T)$ is the probability estimator that maximizes the likelihood of $\mathbf{s}^T$, we have  $N_{s}/T = \hat{P}^*_{\mathcal{F}_\Pi}(s)$ and $N_{s,s'}/N_{s} = \hat{P}_{\mathcal{F}_\Pi}^*(s'|s)$, since $\mathbf{s}^T$ was generated following $P_{\Pi|\Pi'}$. 
Thus: 
\begin{align}\label{eq:entropyF}
\frac{1}{T} (H^{T,*}_{\mathcal{F}_{\Pi'}}(\mathbf{s}^T) - H^{T,*}_{\mathcal{F}_\Pi}(\mathbf{s}^T)) 
&=\sum_{s \in \mathbf{s}^T} \hat{P}^*_{\mathcal{F}_\Pi}(s) \sum_{s' \in \mathbf{s}^T} \hat{P}_{\mathcal{F}_\Pi}^*(s'|s) \log [\hat{P}_{\mathcal{F}_\Pi}^*(s'|s)/ \hat{P}_{\mathcal{F}_{\Pi'}}^*(s'|s)] \nonumber \\
&=  D(\hat{P}^*_{\mathcal{F}_\Pi} || \hat{P}^*_{\mathcal{F}_{\Pi'}}).
\end{align}

Next, as  $H^{T,*}_{\mathcal{F}_\Pi}(\mathbf{s}^T)  = - \log(\hat{L}^{T,*}_{\mathcal{F}_\Pi}(\mathbf{s}^T))$ and $H^{T,*}_{\mathcal{F}_{\Pi'}}(\mathbf{s}^T)  = - \log(\hat{L}^{T,*}_{\mathcal{F}_{\Pi'}}(\mathbf{s}^T))$, by using  \eqref{eq:entropyF} and Proposition \ref{prop:likelihood_global_fsm},  we obtain: 

\begin{align}
\frac{1}{T} (-\text{log}(\widehat{\mathcal{L}}^T_{-,\Pi'}(z,t)) +\text{log}(\widehat{\mathcal{L}}^T_{-,\Pi}(z,t)))
=  D(\hat{P}^*_{\mathcal{F}_\Pi} || \hat{P}^*_{\mathcal{F}_{\Pi'}}).
\end{align}

And thus, 
\begin{align}
\label{eq:lik_entropy}
\frac{1}{T} (-\text{log}(\widehat{\mathcal{L}}^T_{\sim,\Pi'}(z,t)) +\text{log}(\widehat{\mathcal{L}}^T_{\sim,\Pi}(z,t)))
&=   D(\hat{P}^*_{\mathcal{F}_\Pi} || \hat{P}^*_{\mathcal{F}_{\Pi'}}) \nonumber \\
& + \frac{1}{T} (\text{log}(\widehat{\mathcal{L}}^T_{-,\Pi'}(z,t)) -\text{log}(\widehat{\mathcal{L}}^T_{\sim,\Pi'}(z,t))) \nonumber \\
& - \frac{1}{T} (\text{log}(\widehat{\mathcal{L}}^T_{-,\Pi}(z,t)) -\text{log}(\widehat{\mathcal{L}}^T_{\sim,\Pi}(z,t))).
\end{align}

Now from \eqref{eq:lik_entropy} and Proposition \ref{prop:approx_likelihood_max}, we deduce that the quantity $\frac{1}{T} ( H^T_{\Pi'}(z,t) -  H^T_{\Pi}(z,t))$ converges toward $D(\hat{P}^*_{\mathcal{F}_\Pi} || \hat{P}^*_{\mathcal{F}_{\Pi'}})$ in probability when $T$ goes toward infinity.  So does $\frac{1}{T} ( C^T_{\Pi'}(z,t) - C^T_{\Pi}(z,t)$, with  $C^T_{\Pi}(z,t) =  H^T_{\Pi}(z,t) + \gamma m \text{log} (n)$,  
as the term 
$\frac{\gamma m \text{log} (n)}{T}$ goes toward 0 as $T$ goes toward infinity, as $\gamma$ and $m$ are  positive constants 
and  $n \leq T$. 

Since the sequence  $\mathbf{s}^T$ is generated 
with probability $P^*_{\mathcal{F}_{\Pi}}$, we have $D(    P^*_{\mathcal{F}_{\Pi}} || \hat{P}^*_{\mathcal{F}_\Pi} ) \rightarrow 0$ 

as $T \rightarrow \infty$. It thus remain to demonstrate that  $D(P^*_{\mathcal{F}_{\Pi}}||\hat{P}^*_{\mathcal{F}_{\Pi'}})>0$ to conclude the proof.

Since $\Pi'$ is, by 
assumption different from $\Pi$, and thus following Theorem 1 incompatible with $P_\Pi$, no valid assignment of parameters for ${\cal F}_{\Pi'}$ can generate $P_{\Pi|\Pi'}$, and, thus, $P^*_{{\cal F}_\Pi}$ is not in the set  of valid parameters $V(\Pi') \subset V^*$ for ${\cal F}^*_{\Pi'}$, with $V^*$ the whole set of valid parameters in ${\cal F}^*$. 
Said otherwise,   we can distinguish two possible settings: 
\begin{itemize}
    \item At least two symbols  $z$ and $z'$  belong to the same sub-alphabet $A_{e_\Pi}$ in $\Pi$ and to two respective different sub-alphabets $A_{e_{\Pi'}}$ and $A_{e'_{\Pi'}}$ in $\Pi'$. For any $s \in {\cal S}^*$, such that $s_\Pi^{e_\Pi}=(z',d')$, with $s_\Pi=g_\Pi(s)$, and $s_{\Pi'}^{e_{\Pi'}}=(z,d)$, with $s_{\Pi'}=g_{\Pi'}(s)$ (i.e., we have $d>d'$),  $D(P^*_{\mathcal{F}_{\Pi}}(.|s) || \hat{P}^*_{\mathcal{F}_{\Pi'}}(.|s))=0$ implies  $P^*_{{\cal F}_\Pi}(a_{\Pi}^{e_\Pi}=z|s) = P^*_{{\cal F}_{\Pi'}}(a_{\Pi'}^{e_{\Pi'}}=z|s) $ for any $(a_{\Pi},a_{\Pi'}) \in A^*_\Pi \times A^*_{\Pi'}$ such that $P^*_{{\cal F}_{\Pi}}(a_{\Pi}|s)>0$. This cannot be true for every such state. Two different states $(s,s') \in {{\cal S}^*}^2$ can be equal on $e_\Pi$ (i.e, $s_\Pi^{e_\Pi}  = {s'}_\Pi^{e_{\Pi}} $), while having a different value for $e_{\Pi'}$ (e.g., $s_{\Pi'}^{e_{\Pi'}}=(z,d)$ and ${s'}_{\Pi'}^{e_{\Pi'}}=(z,d')$, with $d \neq d'$).  From the definition of the probabilities of actions, which are based on the instantaneous risk of  emission, we have: $P^*_{{\cal F}_{\Pi'}}(a_{\Pi'}^{e_{\Pi'}}|s) \neq  P^*_{{\cal F}_{\Pi'}}(a_{\Pi'}^{e_{\Pi'}}|s') $.  At the same time, since $P^*_{{\cal F}_{\Pi}}(a_{\Pi}^{e_{\Pi}}|s) = P^*_{{\cal F}_{\Pi}}(a_{\Pi}^{e_{\Pi}}|s')$, we can conclude that there exists some state $s \in {\cal S}^*$, such that  $P^*_{{\cal F}_\Pi}(a_{\Pi}^{e_\Pi}|s) \neq P^*_{{\cal F}_{\Pi'}}(a_{\Pi'}^{e_{\Pi'}}|s) $.
    \item At least two symbols  $z$ and $z'$  belong to two respective  different sub-alphabets $A_{e_{\Pi}}$ and $A_{e'_{\Pi}}$ in $\Pi$ and to  the same sub-alphabet $A_{e_{\Pi'}}$ in $\Pi'$. Same manner as above,   $P^*_{{\cal F}_\Pi}(a_{\Pi}^{e_\Pi}=z|s)$ cannot be equal to $P^*_{{\cal F}_{\Pi'}}(a_{\Pi'}^{e_{\Pi'}}=z|s) $ for all $s$ such that $s_\Pi^{e_\Pi}=(z,d)$, with $s_\Pi=g_\Pi(s)$, and $s_{\Pi'}^{e_{\Pi'}}=(z',d')$, with $s_{\Pi'}=g_{\Pi'}(s)$. 
\end{itemize}

Thus, $D(P^*_{\mathcal{F}_{\Pi}}||\hat{P}^*_{\mathcal{F}_{\Pi'}})>0$ for any estimation of probabilities in $\mathcal{F}_{\Pi'}$, since  $D(P^*_{\mathcal{F}_{\Pi}}||\hat{P}^*_{\mathcal{F}_{\Pi'}})=0$ only if $\forall s \in {\cal S}^*$ such that $P^*_{\mathcal{F}_{\Pi}}(s)>0$  we have $D(P^*_{\mathcal{F}_{\Pi}}(.|s) || \hat{P}^*_{\mathcal{F}_{\Pi'}}(.|s))=0$, which is impossible for  $\Pi \neq \Pi'$ as demonstrated above. Thus
\begin{equation}
 \underset{T \rightarrow \infty}{\text{lim}} \mathbb{P}_{(z,t) \sim P_{\Pi}}(
        \frac{1}{T} ( C^T_{\Pi'}(z,t) - C^T_{\Pi}(z,t)) >  0) = 1,
\end{equation}
\noindent which concludes the proof of Theorem 2. 

\end{proof}

 \section{Details of the procedures used in the memetic algorithm for alphabet partitioning\label{appendix_MAAP}}

In this section, we describe in more detail the local search (TabuAP) and the crossover
(GLPX) procedures used in Algorithm 1 (see Section III-C of the main paper).

\subsection{TabuAP local search procedure}

TabuAP for alphabet partitioning is adapted from \cite{MAAP_2023}. Starting  from an initial partition $\Pi$, it then iteratively replaces the current solution $\Pi$ by a neighboring solution $\Pi^{'}$ taken from its one-move neighborhood $N(\Pi)$ until it reaches a maximum of $nb_{iter}$ iterations.

A neighboring solution of a partition $\Pi = \cup_{e=1}^m A_e$ is generated by using the \textit{one-move} operator, which displaces a symbol $a \in A_e$ to a different sub-alphabet $A_f$, such that $f \neq e$. A symbol $a \in A_e$ is allowed to be transferred to an existing group $A_f$ for $f=1,\dots,m$, with $f \neq e$, or to a new group $A_{m+1}$. We use $\Pi\ \oplus <a,A_e,A_f>$ to denote the resulting neighboring partition after performing a one-move operation. 

At each iteration, TabuAP  examines the neighborhood and selects a  best admissible neighboring solution $\Pi'$ to replace $\Pi$. A neighboring solution $\Pi\ \oplus <a,A_e,A_f>$ built from $\Pi$ is admissible if the associated one-move $<a, A_e, A_f>$ was not registered in the tabu list, which records the most recently performed moves. Each time a move is performed, its reversed move is added in the tabu list and forbidden during the $t = r(10) + \alpha |\mathcal{A}|$ next iterations (tabu tenure) where $r$ is a  random number uniformly drawn in $1,...,10$ and $\alpha$ is a hyperparameter of the algorithm set to  $0.6$. 

In order to identify the best admissible partition in the neighborhood, we computer the global penalized scores given by (9) of all the neighboring partitions $\Pi\ \oplus <a,A_e,A_f>$ and retain the 
neighboring partition with the lowest score. For a move $<a,A_e,A_f>$ applied to the current partition $\Pi$ and resulting in a new partition $\Pi' = \Pi\ \oplus <a,A_e,A_f>$,  only the negative log-likelihood scores of the changing groups $A_e$ and $A_f$ need to be reevaluated.

\subsection{Greedy Likelihood-based Crossover Operator}

The GLPX crossover, adapted from \cite{MAAP_2023}, aims to transmit to the offspring the largest possible $A_e$ sub-alphabets from both parents with as the lowest entropy as possible,  because our problem is to minimize the global negative log-likelihood of the partition over the whole alphabet.

A GLPX score associated to a group of letters $A_e$ is defined as 
$\frac{H^T(z^e,t^e)}{|A_e|}$, where $H^T(z^e,t^e)$ is the negative log likelihood for transmitter $e$ given by (4) and $|A_e|$ is the size of the sub-alphabet $A_e$.

Given two parent partitions $\Pi_1$ and $\Pi_2$, the GLPX procedure performs two steps. First, it transmits to the offspring the sub-alphabet  with the lowest GLPX score. After having withdrawn the letters of this sub-alphabet in both parents and recomputed all scores, it transmits to the offspring the sub-alphabet  with the lowest GLPX score of the second parent. This procedure is repeated until all the letters  of the alphabet $\mathcal{A}$ are assigned to the offspring. Note that this crossover is asymmetrical. Starting the procedure with parent $\Pi_1$ or parent $\Pi_2$ can produce different offspring partitions. Therefore when used in the MAAP algorithm described in Algorithm 1 to generate two new offspring partitions  $\Pi_{1} = GLPX(\Pi_{1}', \Pi_{2}', (z,t))$ and $\Pi_{2} = GLPX(\Pi_{2}', \Pi_{1}',(z,t))$, the two offspring partitions  $\Pi_{1}$ and $\Pi_{2}$ can be very different (in the sense of edit distance).



\end{document}